\documentclass[twoside,11pt]{article}

\usepackage{jmlr2e_arxiv}
\usepackage{amsmath}
\usepackage{algorithm2e}
\usepackage{upgreek}
\usepackage{hyperref}
\usepackage{amstext}

\makeatletter
\newtheorem{rep@theorem}{\rep@title}
\newcommand{\newreptheorem}[2]{%
\newenvironment{rep#1}[1]{%
 \def\rep@title{#2 \ref{##1}}%
 \begin{rep@theorem}}%
 {\end{rep@theorem}}}
\makeatother

\newreptheorem{theorem}{Theorem}
\newreptheorem{lemma}{Lemma}
\newreptheorem{corollary}{Corollary}
\newreptheorem{proposition}{Proposition}

\newcommand{\set}[1]{\{#1\}} % set
\newcommand{\ignore}[1]{}

\DeclareMathOperator*{\E}{\mathbb E}
\renewcommand{\P}{\mathbb P}
\newcommand{\I}{\mathbf 1}

\newcommand{\R}{\mathfrak R}
\newcommand{\sR}{\R^\text{seq}}
\DeclareMathOperator{\VCdim}{VC-dim}
\DeclareMathOperator{\disc}{disc} % discrepancy

\newcommand{\cA}{\mathcal A}
\newcommand{\cL}{\mathcal L}
\newcommand{\cX}{\mathcal X}
\newcommand{\cY}{\mathcal Y}
\newcommand{\cZ}{\mathcal Z}
\newcommand{\cF}{\mathcal F}
\newcommand{\cG}{\mathcal G}
\newcommand{\cD}{\mathcal D}

\newcommand{\cH}{\mathcal H}
\newcommand{\cN}{\mathcal N}
\newcommand{\cQ}{\mathcal Q}
\newcommand{\bK}{\mathbf K}

\newcommand{\bZ}{\mathbf Z}
\newcommand{\bY}{\mathbf Y}
\newcommand{\bX}{\mathbf X}
\newcommand{\bq}{\mathbf q}
\newcommand{\bp}{\mathbf p}
\newcommand{\bz}{\mathbf z}
\newcommand{\bx}{\mathbf x}
\newcommand{\by}{\mathbf y}
\newcommand{\bw}{\mathbf w}
\newcommand{\bu}{\mathbf u}
\newcommand{\bd}{\mathbf d}
\newcommand{\br}{\mathbf r}
\newcommand{\bv}{\mathbf v}
\newcommand{\bc}{\mathbf c}
\newcommand{\be}{\boldsymbol\epsilon}
\newcommand{\bg}{\boldsymbol\gamma}
\newcommand{\ba}{\boldsymbol\alpha}
\newcommand{\bs}{\boldsymbol\sigma}
\newcommand{\bP}{\mathbf P}

\newcommand{\tl}{\widetilde }
\newcommand{\hh}{\widehat }

\newcommand{\Rset}{\mathbb R} % reals
\newcommand{\e}{\epsilon}  % epsilon
\renewcommand{\d}{\delta}  % delta
\newcommand{\s}{\sigma}    % sigma

\ShortHeadings{Theory and Algorithms for Forecasting Time Series}{Kuznetsov and Mohri}
\firstpageno{1}

\begin{document}

\title{Theory and Algorithms for Forecasting Time Series}

\author{\name Vitaly Kuznetsov \email vitaly@cims.nyu.edu\\
        \addr Google Research\\
        \ignore{New York University\\}
        New York, NY 10011, USA
        \AND
        \name Mehryar Mohri \email mohri@cims.nyu.edu \\
        \addr Courant Institute and Google Research\\
        \ignore{New York University\\}
        New York, NY 10011, USA}

\editor{}

\maketitle

\begin{abstract}%   <- trailing '%' for backward compatibility of .sty file
  We present data-dependent learning bounds for the general scenario
  of non-stationary non-mixing stochastic processes.  Our learning
  guarantees are expressed in terms of a data-dependent measure of
  sequential complexity and a discrepancy measure that can be
  estimated from data under some mild assumptions.
  We also also provide novel analysis of stable time series forecasting
  algorithm using this new notion of discrepancy that we introduce.
  We use our learning bounds to devise new algorithms for non-stationary time
  series forecasting for which we report some preliminary experimental
  results. An extended abstract has appeared in \citep{KuznetsovMohri2015}.
\end{abstract}


\section{Introduction}
\label{sec:intro}

Time series forecasting plays a crucial role in a number of domains
ranging from weather forecasting and earthquake prediction to
applications in economics and finance. The classical statistical
approaches to time series analysis are based on generative models such
as the autoregressive moving average (ARMA) models, or their
integrated versions (ARIMA) and several other extensions
\citep{Engle1982,Bollerslev1986,BrockwellDavis1986,BoxJenkins1990,Hamilton1994}.
Most of these models rely on strong assumptions about the noise terms,
often assumed to be i.i.d.\ random variables sampled from a Gaussian
distribution, and the guarantees provided in their support are only
asymptotic.

An alternative non-parametric approach to time series analysis
consists of extending the standard i.i.d.\ statistical learning theory
framework to that of stochastic processes. In much of this work, the
process is assumed to be stationary and suitably mixing
\citep{Doukhan1994}. Early work along this approach consisted of the
VC-dimension bounds for binary classification given by
\cite{Yu1994}\ignore{based on independent block technique due to
  Bernstein } under the assumption of stationarity and $\beta$-mixing.
Under the same assumptions, \cite{Meir2000} presented bounds in terms
of covering numbers for regression losses and
\cite{MohriRostamizadeh2009} proved general data-dependent Rademacher
complexity learning bounds. \cite{Vidyasagar1997} showed that PAC
learning algorithms in the i.i.d.\ setting preserve their PAC learning
property in the $\beta$-mixing stationary scenario.  A similar result
was proven by \cite{ShaliziKontorovich2013} for mixtures of
$\beta$-mixing processes and by \cite{BertiRigo1997} and
\cite{Pestov2010} for exchangeable random variables.
\cite{AlquierWintenberger2010} and \cite{AlquierLiWintenberger2014}
also established PAC-Bayesian learning guarantees under weak
dependence and stationarity.

A number of algorithm-dependent bounds have also been derived for the
stationary mixing setting. 
%\ignore{
\cite{LozanoKulkarniSchapire2006} studied
the convergence of regularized boosting.
%}
 \cite{MohriRostamizadeh2010}
gave data-dependent generalization bounds for stable algorithms for
$\varphi$-mixing and $\beta$-mixing stationary processes.
\cite{SteinwartChristmann2009} proved fast learning rates for
regularized algorithms with $\alpha$-mixing stationary sequences and
\cite{ModhaMasry1998} gave guarantees for certain classes of models
under the same assumptions.

However, stationarity and mixing are often not valid assumptions. For
example, even for Markov chains, which are among the most widely used
types of stochastic processes in applications, stationarity does not
hold unless the Markov chain is started with an equilibrium
distribution. Similarly, long memory models such as ARFIMA, may not be
mixing or mixing may be arbitrarily slow \citep{Baillie1996}.  In
fact, it is possible to construct first order autoregressive processes
that are not mixing \citep{Andrews1983}.  Additionally, the mixing
assumption is defined only in terms of the distribution of the
underlying stochastic process and ignores the loss function and the
hypothesis set used. This suggests that mixing may not be the right
property to characterize learning in the setting of stochastic
processes.

A number of attempts have been made to relax the assumptions of
stationarity and mixing. \cite{AdamsNobel2010} proved asymptotic
guarantees for stationary ergodic sequences. \cite{AgarwalDuchi2013}
gave generalization bounds for asymptotically stationary (mixing)
processes in the case of stable on-line learning algorithms.
\cite{KuznetsovMohri2014} established learning guarantees for fully
non-stationary $\beta$- and $\varphi$-mixing processes.

In this paper, we consider the general case of non-stationary
non-mixing processes.  We are not aware of any prior work providing
generalization bounds in this setting. In fact, our bounds appear to
be novel even when the process is stationary (but not mixing).  The
learning guarantees that we present hold for both bounded and
unbounded memory models. Deriving generalization bounds for unbounded
memory models even in the stationary mixing case was an open question
prior to our work \citep{Meir2000}.  Our guarantees cover the majority
of approaches used in practice, including various autoregressive and
state space models.

The key ingredients of our generalization bounds are a data-dependent
measure of sequential complexity (\emph{expected sequential covering number}
or \emph{sequential Rademacher complexity} \citep{RakhlinSridharanTewari2010})
and a measure of \emph{discrepancy} between the sample and target
distributions.  \ignore{Sequential complexities has been shown to
  characterize learning in the on-line learning setting
  \citep{RakhlinSridharanTewari2010,RakhlinSridharanTewari2011b,RakhlinSridharanTewari2015}.}
\cite{KuznetsovMohri2014,KuznetsovMohri2016} also give generalization bounds in terms of
discrepancy. However, unlike these result,
our analysis does not require any mixing assumptions which are hard to
verify in practice.\ignore{The presence of the discrepancy measure in
  our bounds also indicates a connection between forecasting
  non-stationary time series and the domain adaption problem
  \citep{KiferBenDavidGehrke2004,BenDavidBlitzerCrammerPereira2007,BlitzerCrammerKuleszaPereiraWortman2008,MansourMohriRostamizadeh2009,MansourMohriRostamizadeh2009b}
  and drifting scenario
  \citep{BenDavidBenedekMansour1989,Bartlett1992,BarveLong1996,MohriMunozMedina2012}.
  In these settings, it is often desirable to estimate discrepancy
  from data. } More importantly, under some additional mild
assumption, the discrepancy measure that we propose can be estimated
from data, which leads to data-dependent learning guarantees for
non-stationary non-mixing case.

We also show that our methodology can be applied to analysis of stable
time series forecasting algorithms, which extends previous results of
\cite{MohriRostamizadeh2010} to the setting of non-stationary non-mixing
stochastic processes. Proof techniques combining decoupled tangent sequences
with stability analysis can be of independent interest.

We devise new algorithms for non-stationary time series forecasting
that benefit from our data-dependent guarantees. The parameters of
generative models such as ARIMA are typically estimated via the
maximum likelihood technique, which often leads to non-convex
optimization problems. In contrast, our objective is convex and leads
to an optimization problem with a unique global solution that can be
found efficiently. Another issue with standard generative models is
that they address non-stationarity in the data via a
\emph{differencing} transformation which does not always lead to a
stationary process. In contrast, we address the problem of
non-stationarity in a principled way using our learning
guarantees.

The rest of this paper is organized as follows.  The formal definition
of the time series forecasting learning scenario as well as that of
several key concepts is given in Section~\ref{sec:prelim}. In
Section~\ref{sec:bounds}, we introduce and prove our new
generalization bounds. Section~\ref{sec:kernel-classes} provides
an analysis in the special case of kernel-based hypotheses with
regression losses. Section~\ref{sec:stability} provides an analysis
of regularized ERM algorithms based on the notion of algorithmic
stability. In Section~\ref{sec:discrepancy}, we give
data-dependent learning bounds based on the empirical
discrepancy. These results are used to devise new
forecasting algorithms in Section~\ref{sec:algo}. In
Appendix~\ref{sec:experiments}, we report the results of preliminary
experiments using these algorithms.

\section{Preliminaries}
\label{sec:prelim}

We consider the following general time series prediction setting where
the learner receives a realization $(X_1, Y_1), \ldots, (X_T, Y_T)$ of
some stochastic process, with $(X_t, Y_t) \in \cZ = \cX \times \cY$.
The objective of the learner is to select out of a specified family
$H$ a hypothesis $h\colon \cX \to \cY$ that achieves a small
\emph{path-dependent} generalization error
\begin{align}
\label{eq:path-error}
\cL_{T+1}(h, \bZ_1^T) =\E[L(h(X_{T+1}), Y_{T + 1}) | Z_1, \ldots, Z_T]
\end{align}
conditioned on observed data, where
$L\colon \cY \times \cY \to [0, \infty)$ is a given loss function.
The path-dependent generalization error that we consider in
this work is a finer measure of the generalization ability than the
\emph{averaged} generalization error
\begin{align}
\label{eq:avg-error}
\cL_{T+1}(h) = \E[L(h(X_{T+1}), Y_{T + 1})] =\E[\E[L(h(X_{T+1}), Y_{T + 1}) | Z_1, \ldots, Z_T]]
\end{align}
since it only takes into consideration the realized history of the
stochastic process and does not average over the set of all possible
histories. The results that we present in this paper also apply to the
setting where the time parameter $t$ can take non-integer values and
prediction lag is an arbitrary number $l \geq 0$. That is, the error
is defined by $\E[L(h(X_{T+l}), Y_{T + l})| Z_1, \ldots, Z_T]$ but for
notational simplicity we set $l = 1$\ignore{ in the sequel}.

Our setup covers a larger number of scenarios commonly used in
practice. The case $\cX = \cY^p$ corresponds to a large class of
autoregressive models.  Taking $\cX = \cup_{p=1}^\infty \cY^p$ leads
to growing memory models which, in particular, include state space
models.  More generally, $\cX$ may contain both the history of the
process $\set{Y_t}$ and some additional side information. Note that
output space $\cY$ may also be high-dimensional. This covers
both the case when we are trying to forecast high-dimensional
time series as well as multi-step forecasting.

To simplify the notation,\ignore{ in the rest of the paper,} we will often
use the
shorter notation $f(z) = L(h(x), y)$, for any $z = (x, y) \in \cZ$ and
introduce the family $\cF = \set{(x, y) \to L(h(x), y) \colon h \in
  H}$ containing such functions $f$. We will assume a bounded
loss function, that is $|f| \leq M$ for all $f \in \cF$ for some $M
\in \Rset_+$. Finally, we will use the shorthand $\bZ_a^b$ to denote a
sequence of random variables $Z_a, Z_{a+1}, \ldots, Z_b$.

The key quantity of interest in the analysis of generalization is the
following supremum of the empirical process defined as follows:
\begin{equation}
\label{eq:sup-emp-process}
\Phi(\bZ_1^T) = \sup_{f \in \cF} \Bigg( \E[f(Z_{T + 1}) | \bZ_1^T]
- \sum_{t = 1}^T q_t f(Z_t)
\Bigg),
\end{equation}
where $q_1, \ldots, q_T$ are real numbers, which in the standard
learning scenarios are chosen to be uniform. In our general setting,
different $Z_t$s may follow different distributions, thus distinct
weights could be assigned to the errors made on different sample
points depending on their relevance to forecasting the future $Z_{T +
  1}$. The generalization bounds that we present below are for an
arbitrary sequence $\bq = (q_1, \ldots q_T)$ which, in particular,
covers the case of uniform weights. Remarkably, our bounds do not even
require the non-negativity of $\bq$.  \ignore{As shown later
(Section~\ref{sec:algo}), this added flexibility also helps us derive
new algorithms for forecasting non-stationary time series.}

The two key ingredients of our analysis are sequential complexities
\citep{RakhlinSridharanTewari2010} and a novel
discrepancy measure between target and source distributions. In the next
two sections we provide a detailed overview of these notions.

\subsection{Sequential Complexities}
\label{sec:seq_complexities}

Our generalization bounds are expressed in terms of data-dependent
measures of sequential complexity such as expected sequential covering
number or sequential Rademacher complexity
\citep{RakhlinSridharanTewari2010}, which we review in this section.
\ignore{  We give a brief overview of the
notion of sequential covering number and refer the reader to the
aforementioned reference for further details.}

We adopt the following
definition of a complete binary tree: a $\cZ$-valued complete binary
tree $\bz$ is a sequence $(z_1, \ldots, z_T)$ of $T$ mappings
$z_t \colon \set{\pm 1}^{t - 1} \to \cZ$, $t \in [1, T]$.  A path in
the tree is $\s = (\s_1, \ldots, \s_{T - 1}) \in \set{\pm 1}^{T-1}$.
To simplify the
notation we will write $z_t(\bs)$ instead of
$z_t(\s_1, \ldots, \s_{t - 1})$, even though $z_t$ depends only on the
first $t - 1$ elements of $\bs$. The following definition generalizes
the classical notion of covering numbers to sequential setting.  A set
$V$ of $\Rset$-valued trees of depth $T$ is a \emph{sequential
  $\alpha$-cover} (with respect to $\bq$-weighted $\ell_p$ norm) of a
function class $\cG$ on a tree $\bz$ of depth $T$ if for all
$g \in \cG$ and all $\bs \in \set{\pm}^T$, there is $\bv \in V$ such
that
\begin{equation*}
\Bigg( \sum_{t = 1}^T \big| \bv_t(\bs) -  g(\bz_t(\bs)) \big|^p
\Bigg)^{\frac 1 p}
\leq \| \bq \|_q^{-1} \alpha,
\end{equation*}
where $\|\cdot\|_q$ is the dual norm.
The \emph{(sequential) covering number} $\cN_p(\alpha,\cG,\bz)$
of a function class $\cG$ on a given tree $\bz$ is defined to be
the size of the minimal sequential cover. The \emph{maximal covering number}
is then taken to be $\cN_p(\alpha,\cG) = \sup_{\bz} \cN_p(\alpha,\cG,\bz)$. 
One can check that in the case of uniform weights this definition
coincides with the standard definition of sequential covering
numbers. Note that this is a purely combinatorial notion of complexity
which ignores the distribution of the process in the given learning
problem.

Data-dependent sequential covering numbers can be defined as follows.
Given a stochastic process distributed according to the distribution
$\bp$ with $\bp_t(\cdot|\bz_1^{t-1})$ denoting the conditional
distribution at time $t$, we sample a $\cZ \times \cZ$-valued tree of
depth $T$ according to the following procedure. Draw two independent
samples $Z_1, Z'_1$ from $\bp_1$: in the left child of the root draw
$Z_2,Z'_2$ according to $\bp_2(\cdot|Z_1)$ and in the right child
according to $\bp_2(\cdot|Z'_2)$. More generally, for a node that can
be reached by a path $(\s_1, \ldots, \s_t)$, we draw $Z_t,Z'_t$
according to $\bp_t(\cdot|S_1(\s_1), \ldots, S_{t-1}(\s_{t-1}))$,
where $S_t(1) = Z_t$ and $S_t(-1) = Z'_t$. Let $\bz$ denote the tree
formed using $Z_t$s and define the \emph{expected covering number} to
be $\E_{\bz \sim T(\bp)}[\cN_p(\alpha, \cG, \bz)]$, where $T(\bp)$ denotes
the distribution of $\bz$. For i.i.d.\ sequences expected sequential covering
numbers exactly coincide with the notion of expected covering numbers
from classical statistical learning theory.

The sequential
Rademacher complexity of a function class $\cZ$ is defined as the following:
\begin{align}
\label{eq:seq-rademacher}
\sR_T(\cG) =  \sup_{\bz}
\E \Bigg[ \sup_{g \in \cG}  \sum_{t = 1}^T \s_t q_t g( z_t(\bs) ) 
\Bigg],
\end{align}
where the supremum is taken over all complete binary trees of depth
$T$ with values in $\cZ$ and where $\bs$ is a sequence of Rademacher
random variables.

Similarly, one can also define
\emph{distribution-dependent} sequential Rademacher complexity
as well as other notions of sequential complexity such
as \emph{Littlestone dimension} and \emph{sequential metric entropy} that
have been shown to characterize learning in on-line framework.
For further details, we refer the reader to
\citep{Littlestone1987,RakhlinSridharanTewari2010,RakhlinSridharanTewari2011,RakhlinSridharanTewari2015,RakhlinSridharanTewari2015b}

\ignore{
In a similar manner, one can define other measures of complexity such
as sequential Rademacher complexity and the Littlestone dimension
\citep{RakhlinSridharanTewari2015} as well as their data-dependent
counterparts \citep{RakhlinSridharanTewari2011}.}

\subsection{Discrepancy Measure}
\label{sec:disc}

The final ingredient needed for expressing our learning guarantees is
the notion of \emph{discrepancy} between target distribution and the
distribution of the sample:
\begin{align}
\label{eq:discrepancy}
\disc(\bq) = \sup_{f \in \cF} \bigg( \E [f(Z_{T + 1})|\bZ_1^T] -
\sum_{t = 1}^T q_t \E [f(Z_{t}) | \bZ_1^{t-1}] \bigg).
\end{align}
The discrepancy $\disc$ is a natural measure of the non-stationarity
of the stochastic process $\bZ$ with respect to both the loss function
$L$ and the hypothesis set $H$. In particular, note that if the
process $\bZ$ is i.i.d., then we simply have $\disc(\bq) = 0$ provided
that $q_t$s form a probability distribution.

As a more insightful example, consider the case of
a time-homogeneous Markov chain on a set $\set{0, \ldots, N-1}$
such that $\bP(X_t \equiv (i-1)\text{ mod } N | X_{t-1} = i) = p$ and
$\bP(X_t \equiv (i+1)\text{ mod } N  | X_{t-1} = i) = 1-p$
for some $0\leq p \leq 1$. This process is non-stationary if it
is not started with an equilibrium distribution. Suppose
that the set of hypothesis is $\set{x \mapsto a(x-1) + b(x+1) \colon
a + b = 1, a,b \geq 0}$ and the loss function $L(y,y') = \ell(|y-y'|)$
for some $\ell$. It follows that for any $(a,b)$
\begin{align*}
\E [f(Z_{t}) | \bZ_1^{t-1}] = p|a-b - 1| + (1-p)|a-b + 1|
\end{align*}
and hence $\disc(\bq) = 0$ provided $\bq$ is a probability distribution.
Note that if we chose a larger hypothesis set
$\set{x \mapsto a(x-1) + b(x+1) \colon a,b \geq 0}$ then
\begin{align*}
\E [f(Z_{t}) | \bZ_1^{t-1}] = p|(a+b-1)X_t + a-b| + (1-p)|(a+b+1)X_t + a-b|
\end{align*}
and in general it may be the case that $\disc(\bq) \neq 0$. This highlights
an important property of discrepancy: it takes into account not only
the underlying distribution of the stochastic process but other
components of the learning problem such as the loss function and
the hypothesis set that is being used.

The weights $\bq$ play a crucial role in the learning problem as well.
Consider our earlier example, where transition probability distributions
$(p_i, 1-p_i)$ are different for each state $i$.
Note that choosing $\bq$ to be a uniform distribution, in general,
leads to a non-zero discrepancy. However, with $q'_t = \I_{X_{t-1} = X_T}$
and $q_t = q'_t / \sum_{t=1}^T q'_t$ discrepancy is zero. Note that
in fact it is not the only choice that leads to a zero discrepancy
in this example and in fact any distribution that is supported
on $t$s for which $X_{t-1} = X_{T}$ will lead to the same result.
However, $q_t$s based on $q'_t$ have an advantage of providing
the largest effective sample.

Similar results can be established for weakly stationary
stochastic process as well.

\ignore{ As a
more insightful example, consider the case of a weakly stationary
process $\bY_1^T$.  Recall that the process is weakly stationary if
$\E[Y_t]$ does not depend on $t$ and there is a function $\rho$ such
that $\E[Y_t Y_s] = \rho(t-s)$.  If $L$ is a squared loss, $H =
\set{\by \to \bw \cdot \by}$ is a set of linear hypothesis and $\cZ =
\cY^p \times \cY$ then one can show that $\Delta = 0$. See
Lemma~\ref{lm:weak-stationary} in Appendix~\ref{sec:proofs}.  This
example illustrates the fact that the discrepancy $\Delta$ is
sensitive to both the loss function $L$ and the hypothesis set $H$.
Finally, observe that if $\bq$ is a positive sequence, then the
following upper bounds holds for $\Delta$:
\begin{align}
\label{eq:old-discrepancy}
\Delta \leq \sum_{t = 1}^T q_t \sup_{f \in \cF} \Big|\E [f(Z_{T + 1})] - \E [f(Z_{t})] \Big|,
\end{align}
where the right-hand side is the discrepancy measure used in the
learning bounds given by \citet{KuznetsovMohri2014} for non-stationary
mixing case. This shows that our generalization bounds are in terms of
a finer measure of non-stationarity.}

It is also possible to give bounds on $\disc(\bq)$ in terms of other natural
distances between distribution. For instance, if $\bq$ is a probability
distribution then Pinsker's inequality yields
\begin{align*}\textstyle
\disc(\bq) \leq M \Big\| \bP_{T + 1}(\cdot|\bZ_1^T) \!-\!
\sum_{t = 1}^T q_t \bP_{t}(\cdot|\bZ_1^{t-1}) \Big\|
\leq
\tfrac{1}{\sqrt{2}} D^{\frac{1}{2}} \Bigg(\bP_{T + 1}(\cdot|\bZ_1^T) \parallel
 \sum_{t = 1}^T q_t \bP_{t}(\cdot|\bZ_1^{t-1})\Bigg),
\end{align*}
where $\|\cdot\|$ is the total variation distance and
$D(\cdot \parallel \cdot)$ the relative entropy, $\bP_{t + 1}(\cdot|\bZ_1^t)$
the conditional distribution of $Z_{t+1}$,
and $ \sum_{t = 1}^T q_t \bP_{t}(\cdot|\bZ_1^{t-1})$
the mixture of the sample marginals. Note that these upper bounds
are often too loose since they are agnostic to the loss function
and the hypothesis set that is being used. For our earlier
Markov chain example, the support of $\bP_{T+1}(\cdot|\bZ_1^T)$
is $\set{X_T -1, X_T+1}$ while the mixture
$\frac{1}{T} \sum_{t = 1}^T \bP_{t}(\cdot|\bZ_1^{t-1})$ is likely
to be supported on the whole set $\set{0, \ldots, N-1}$ which leads
to a large total variation distance. Of course, it is possible
to choose $q_t$s so that the mixture is also supported only on
$\set{X_T -1, X_T+1}$ but that may reduce the effective sample size
which is not necessary when working with $\disc(\bq)$.
\ignore{Alternatively, if the target distribution 
at lag $l$, $\bP = \bP_{T + l}$ is a
stationary distribution of an asymptotically stationary process $\bZ$
\citep{AgarwalDuchi2013,KuznetsovMohri2014}, then for $q_t=1/T$ we have
\begin{align*}
\Delta
\leq \frac{M}{T} \sum_{t = 1}^T
\| \bP - \bP_{t+l}(\cdot | \bZ^{t}_{-\infty})\|_\text{TV} \leq
 \upphi(l), 
\end{align*}
where $\upphi(l) =\sup_s \sup_\bz [\| \bP - \bP_{l+s}(\cdot |
\bz^s_{-\infty})\|_\text{TV}]$ is the coefficient of asymptotic
stationarity. The process is asymptotically stationary if
$\lim_{l\to\infty}\upphi(l) = 0$.}

However, the most important property of
the discrepancy $\disc(\bq)$ is that, as shown later in
Section~\ref{sec:discrepancy}, it can be estimated from data under
some additional mild assumptions.
\citep{KuznetsovMohri2014} also give generalization bounds
based on averaged generalization error
for non-stationary mixing processes in terms of a related notion
of discrepancy. It is not known if the discrepancy measure used in
\citep{KuznetsovMohri2014} can be estimated from data.

\section{Generalization Bounds} 
\label{sec:bounds}

In this section, we prove new generalization bounds for forecasting
non-stationary time series. The first step consists of using
\emph{decoupled tangent} sequences to establish concentration results
for the supremum of the empirical process $\Phi(\bZ_1^T)$. Given a
sequence of random variables $\bZ_1^T$ we say that ${{\bZ'}}_1^T$ is a
decoupled tangent sequence if $Z'_t$ is distributed according to
$\P( \cdot | \bZ_1^{t-1})$ and is independent of $\bZ_t^\infty$.  It
is always possible to construct such a sequence of random variables
\citep{DeLaPenaGine1999}.  The next theorem is the main result of this
section.

\begin{theorem}
\label{th:bound}
Let $\bZ_1^T$ be a sequence of random variables distributed according
to $\bp$. Fix $\epsilon > 2 \alpha > 0$. Then, the following holds:
\begin{equation*}
\P\big( \Phi(\bZ_1^T) - \disc(\bq) \geq \e \big)
\leq \E_{\bv \sim T(\bp)}\big[ \cN_1(\alpha, \cF, \bv) \big]
\exp\bigg(\!\!-\! \frac{(\e - 2\alpha)^2}{2 M^2 \| \bq \|_2^2}\bigg).
\end{equation*}
\end{theorem}

\begin{proof}
  The first step is to observe that, since the difference of the
  suprema is upper bounded by the supremum of the difference, it
  suffices to bound the probability of the following event
\begin{align*}
\Bigg\{\sup_{f \in \cF} \Bigg( \sum_{t = 1}^T
q_t (\E[f(Z_{t}) | \bZ_1^{t-1}]  -  f(Z_{t})) \Bigg) \geq \e \Bigg\}.
\end{align*}
By Markov's inequality, for any $\lambda > 0$, the following
inequality holds:
\begin{align*}
& \P\Bigg(\sup_{f \in \cF} \Bigg( \sum_{t = 1}^T
q_t (\E[f(Z_{t}) | \bZ_1^{t-1}]  - f(Z_{t})) \Bigg) \geq \e \Bigg)\\ &\leq
\exp(-\lambda \e)
\E\Bigg[\exp\Bigg( \lambda
\sup_{f \in \cF} \Bigg( \sum_{t = 1}^T
q_t(\E[f(Z_{t}) | \bZ_1^{t-1}]  - f(Z_{t})) \Bigg) \Bigg) \Bigg].
\end{align*}
Since ${\bZ'}_1^T$ is a tangent sequence the following equalities hold:
$\E[f(Z_{t}) | \bZ_1^{t-1}] = \E[f(Z'_{t}) | \bZ_1^{t-1}] =
\E[f(Z'_{t}) | \bZ_1^{T}]$. Using these equalities
and Jensen's inequality, we obtain the following:
\begin{align*}
& \E\bigg[\exp\Big( \lambda
\sup_{f \in \cF} \sum_{t = 1}^T
q_t \big(\E[f(Z_{t}) | \bZ_1^{t-1}]  - f(Z_{t}) \big) \Big) \bigg]
\\ &=
\E\bigg[\exp\Big( \lambda
\sup_{f \in \cF} \E\Big[ \sum_{t = 1}^T
q_t \big( f(Z'_{t})  - f(Z_{t}) \big)  | \bZ_1^{T}\Big] \Big) \bigg]
\\ &\leq
\E\bigg[\exp\Big( \lambda
\sup_{f \in \cF} \sum_{t = 1}^T q_t \big( f(Z'_{t})  - f(Z_{t}) \big) \Big) \bigg],
\end{align*}
where the last expectation is taken over the joint measure of
$\bZ_1^T$ and ${\bZ'}_1^T$. Applying Lemma~\ref{lm:sym}
(Appendix~\ref{sec:proofs}), we can further bound this expectation by
\begin{align*}
&\E_{(\bz, \bz') \sim T(\bp)} \E_\sigma
\bigg[ \exp\bigg( \lambda
\sup_{f \in \cF} \sum_{t = 1}^T
\sigma_t  q_t \Big( f({\bz'}_{t}(\bs))  - f({\bz}_{t}(\bs)) \Big) \bigg) \bigg]
\\ & \leq
\E_{(\bz, \bz') \sim T(\bp)} \E_\sigma
\bigg[ \exp\bigg( \lambda
\sup_{f \in \cF} \sum_{t = 1}^T
\sigma_t  q_t f({\bz'}_{t}(\bs)) + \lambda \sup_{f \in \cF} \sum_{t = 1}^T
-\sigma_t q_t  f({\bz}_{t}(\bs))  \bigg) \bigg]
\\ & \leq
\tfrac{1}{2}
\E_{(\bz, \bz')} \E_\sigma
\bigg[ \exp\bigg( 2 \lambda
\sup_{f \in \cF} \sum_{t = 1}^T
\sigma_t  q_t f({\bz'}_{t}(\bs))  \bigg) \bigg]
 + \tfrac{1}{2}
\E_{(\bz, \bz')} \E_\sigma
\bigg[ \exp\bigg( 2 \lambda
\sup_{f \in \cF} \sum_{t = 1}^T
\sigma_t  q_t f({\bz}_{t}(\bs))  \bigg) \bigg]
\\ &=
\E_{\bz \sim T(\bp)} \E_\sigma
\bigg[ \exp\bigg( 2 \lambda
\sup_{f \in \cF} \sum_{t = 1}^T
\sigma_t  q_t f({\bz}_{t}(\bs))  \bigg) \bigg] 
\ignore{\\ &=
\E_{\bv \sim T(\P)} \E_\sigma
\bigg[ \exp\bigg( \lambda
\sup_{f \in \cF} \sum_{t = 1}^T
\sigma_t q_t g_f(\bv_t(\bs))  \bigg) \bigg]},
\end{align*}
where for the second inequality we used Young's inequality and
for the last equality we used symmetry.
Given $\bz$ let $C$ denote the minimal $\alpha$-cover
with respect to the $\bq$-weighted $\ell_1$-norm of
$\cF$ on $\bz$. Then, the following bound holds
\begin{align*}
\sup_{f \in \cF} \sum_{t = 1}^T
\sigma_t q_t f(\bz_t(\bs)) \leq
\max_{\bc \in C} \sum_{t = 1}^T
\sigma_t q_t \bc_t(\bs) + \alpha.
\end{align*}
\ignore{Therefore, it follows that for a fixed $\bz$,}
By the monotonicity of the exponential function,
\begin{align*}
\E_\sigma
\bigg[ \exp\bigg( 2\lambda
\sup_{f \in \cF} \sum_{t = 1}^T
\sigma_t q_t f(\bz_t(\bs)) \bigg) \bigg]
& \leq \exp(2 \lambda \alpha)
\E_\sigma
\bigg[\exp\bigg( 2\lambda
\max_{\bc \in C} \sum_{t = 1}^T
\sigma_t q_t \bc_t(\bs) \bigg) \bigg]
\\ &\leq
\exp(2\lambda \alpha) \sum_{\bc \in C}
\E_\sigma
\bigg[\exp\bigg( 2\lambda
\sum_{t = 1}^T
\sigma_t q_t \bc_t(\bs) \bigg) \bigg].
\end{align*}
Since $\bc_t(\bs)$ depends only on $\sigma_1, \ldots, \sigma_{T-1}$,
by Hoeffding's bound,
\begin{align*}
\E_\sigma
\bigg[\exp\bigg( 2\lambda
\sum_{t = 1}^T
\sigma_t q_t \bc_t(\bs) \bigg) \bigg]
 & =
\E \bigg[
\exp\bigg( 2\lambda
\sum_{t = 1}^{T-1}
\sigma_t q_t \bc_t(\bs) \bigg)
\E_{\sigma_T}
\bigg[\exp\bigg( 2\lambda
\sigma_T q_T \bc_T(\bs) \bigg) \bigg| \bs_1^{T-1} \bigg] \bigg]
\\ &\leq
\E \bigg[
\exp\bigg( 2\lambda
\sum_{t = 1}^{T-1}
\sigma_t q_t \bc_t(\bs) \bigg)
\exp(2 \lambda^2 q_T^2 M^2) \bigg]
\end{align*}
and iterating this inequality and using the union bound, we obtain the following:
\begin{equation*}
\P\bigg(\sup_{f \in \cF} \sum_{t = 1}^T
q_t (\E[f(Z_{t}) | \bZ_1^{t-1}]  - f(Z_{t})) \geq \e \bigg)\\
\leq \mspace{-10mu} \E_{\bv \sim T(\bp)} \mspace{-10mu} [ \cN_1(\alpha, \cG, \bv) ]
\exp\Big(\!\! -\lambda(\e - 2\alpha) + 2 \lambda^2 M^2 \|\bq\|_2^2 \Big).
\end{equation*}
Optimizing over $\lambda$ completes the proof.  
\end{proof}

An immediate consequence of Theorem~\ref{th:bound} is the following
result.
\begin{corollary}
\label{cor:bound}
For any $\delta > 0$, with probability at least $1 - \delta$, for all
$f \in \cF$ and all $\alpha>0$,
\begin{align*}
\E[f(Z_{T+1})|\bZ_1^T] \leq \sum_{t = 1}^T q_t f(Z_{t}) + \disc(\bq)
+ 2 \alpha + M \| \bq \|_2
\sqrt{2\log\frac{\E_{\bv \sim T(\P)}[ \cN_1(\alpha, \cF, \bv) ]}{\delta}}.
\end{align*}
\end{corollary}
We are not aware of other finite sample bounds in a non-stationary
non-mixing case. In fact, our bounds appear to be novel even in the
stationary non-mixing case.
 
While \cite{RakhlinSridharanTewari2015} give high probability bounds
for a different quantity than the quantity of interest
in time series prediction,
\begin{align}
\label{eq:max_martingale_difference}
\sup_{f \in \cF} \Bigg( \sum_{t = 1}^T
q_t (\E[f(Z_{t}) | \bZ_1^{t-1}]  -  f(Z_{t})) \Bigg),
\end{align}
their analysis of this quantity can also be used in our context to derive
high probability bounds
for $\Phi(\bZ_1^T) - \disc(\bq)$. However, this approach results in bounds
that are in terms of purely combinatorial notions\ignore{ of sequential
complexity} such as maximal sequential covering numbers $\cN_1(\alpha, \cF)$.
While at first sight, this may seem as a minor technical detail, the
distinction is crucial in the setting of time series prediction.
Consider the following example. Let $Z_1$ be drawn from a uniform
distribution on $\set{0, 1}$ and $Z_t \sim p(\cdot | Z_{t-1})$ with
$p(\cdot|y)$ being a distribution over $\set{0,1}$ such that
$p(x | y) = 2/3$ if $x = y$ and $1/3$ otherwise. Let $\cG$ be defined
by $\cG = \set{g(x) = \I_{x \geq \theta} \colon \theta \in [0,1]}$.
Then, one can check that
$\E_{\bv \sim T(\P)}[ \cN_1(\alpha, \cG, \bv) ] = 2$, while
$\cN_1(\alpha, \cG) \geq 2^T$. The data-dependent bounds of
Theorem~\ref{th:bound} and Corollary~\ref{cor:bound} highlight the
fact that the task of time series prediction lies in between
the familiar i.i.d.\ scenario and adversarial on-line learning
setting. 

However, the key component of our learning guarantees is the
discrepancy term $\disc(\bq)$.  Note that in the general non-stationary
case, the bounds of Theorem~\ref{th:bound} may not converge to zero
due to the discrepancy between the target and sample
distributions. This is also consistent with the lower bounds of
\citet{BarveLong1996} that we discuss in more detail in
Section~\ref{sec:discrepancy}.  However, convergence can be
established in some special cases.  In the i.i.d.\ case our bounds
reduce to the standard covering numbers learning guarantees. In the
drifting scenario, with $\bZ_1^T$ being a sequence of independent
random variables, our discrepancy measure coincides with the one used
and studied in \citep{MohriMunozMedina2012}.  Convergence can also be
established for weakly stationary stochastic processes.
However, as we show in Section~\ref{sec:discrepancy}, the most
important advantage of our bounds is that the discrepancy measure we
use can be estimated from data.

We now show that expected sequential covering numbers can be
upper bounded in terms of the sequential Rademacher complexity.
While generalization bounds in terms of sequential Rademacher complexity
are not as tight as bounds in terms expected sequential covering numbers
since the former is a purely combinatorial notion, the analysis
of sequential Rademacher complexity may be simpler for certain
hypothesis classes such as for instance kernel-based hypothesis that
we study in Section~\ref{sec:kernel-classes}.
We have the following extension of Sudakov's
Minoration Theorem to the setting of sequential complexities.

\begin{theorem}
\label{th:sudakov}
The following bound holds:
\begin{align*}
\sup_{\alpha > 0} \frac{\alpha}{2} \sqrt{\log \cN_2(2\alpha, \cF)}
\leq 3 \sqrt{\frac{\pi}{2} \log T}
\sR_T(\cF),
\end{align*}
whenever $\cN_2(2\alpha, \cF) < \infty$.
\end{theorem}

\begin{proof}
We consider the following Gaussian-Rademacher sequential complexity:
\begin{align}
\label{eq:gauss-rademacher}
\mathfrak{G}_T^\text{seq}(\cF, \bz) = \E_{\bg, \bs}\Bigg[
\sup_{f \in \cF} \Big( \sum_{t=1}^T q_t \s_t \gamma_t f(z_t(\bs)) \Big)\Bigg],
\end{align}
where $\bs$ is an independent sequence of Rademacher random variables, $\bg$
is an independent sequence of standard Gaussian random variables and
$\bz$ is a complete binary tree of depth $T$ with values in $\cZ$.

Observe that if $V$ is any $\alpha$-cover with respect to
the $\bq$-weighted $\ell_2$-norm of $\cF$ on $\bz$. Then the following
holds by independence of $\bg$ and $\bs$:
\begin{align*}
\mathfrak{G}^\text{seq}(\cF, \bz) \geq
\E_{\bg} \E_{\bs}\Bigg[
\sup_{\bv \in V} \Big( \sum_{t=1}^T q_t \s_t \gamma_t \bv_t(\bs) \Big)\Bigg]
=
\E_{\bs} \E_{\bg}\Bigg[
\sup_{\bv \in V} \Big( \sum_{t=1}^T q_t \s_t \gamma_t \bv_t(\bs) \Big)\Bigg].
\end{align*}
Observe that $V$ is also  $2\alpha$-cover with respect to
the $\bq$-weighted $\ell_2$-norm of $\cF$ on $\bz$.
We can obtain a smaller $2\alpha$-cover $V_0$ from $V$
be eliminating $\bv$s that are $\alpha$ close to some other $\bv' \in V$.
Since $V$ is finite, let $V = \set{\bv^1, \ldots, \bv^{|V|}}$,
and for each $\bv^i$ we delete $\bv^j \in \set{\bv_{i+1}, \ldots, \bv^{|V|}}$
such that
\begin{align*}
\Bigg( \sum_{t=1}^T \Big(q_t \bv^{i}_t(\bs)
 -  q_t {\bv}^{j}_t(\bs)\Big)^2 \Bigg)^{1/2} \leq
\alpha.
\end{align*}
It is straightforward to verify that $V_0$ is
$2\alpha$-cover with respect to
the $\bq$-weighted $\ell_2$-norm of $\cF$ on $\bz$.
Furthermore, it follows that for a fixed $\bs$, the following
holds: 
\begin{align*}
\E_{\bg}\Bigg[\Big( \sum_{t=1}^T q_t \s_t \gamma_t \bv_t(\bs)
 - \sum_{t=1}^T q_t \s_t \gamma_t {\bv'}_t(\bs)\Big)^2 \Bigg]
\geq \alpha^2.
\end{align*}
for any $\bv',\bv \in V_0$. Let $Z_i, i = 1, \ldots, |V_0|$ be a sequence
of independent Gaussian random variables with $\E[Z_i] = 0$ and
$\E[Z_i^2] = \alpha^2 / 2$. Observe that $\E[(Z_i - Z_j)] = \alpha^2$ and
hence by Sudakov-Fernique inequality it follows that
\begin{align*}
\E_{\bs} \E_{\bg}\Bigg[
\sup_{\bv \in V} \Big( \sum_{t=1}^T q_t \s_t \gamma_t \bv_t(\bs) \Big)\Bigg]
 &\geq
\E_{\bs} \E_{\bg}\Bigg[
\sup_{\bv \in V_0} \Big( \sum_{t=1}^T q_t \s_t \gamma_t \bv_t(\bs) \Big)\Bigg]
\\ &\geq
\E\Big[\max_{i=1,\ldots,|V_0|} Z_i \Big]
\\ &\geq \frac{\alpha}{2} \sqrt{\log |V_0|},
\end{align*}
where the last inequality is the standard result for Gaussian random
variables. Therefore, we conclude that
$\mathfrak{G}^\text{seq}(\cF, \bz) \geq 
\sup_{\alpha > 0} \frac{\alpha}{2} \sqrt{\log \cN_2(2\alpha, \cF, \bz)}$.
On the other hand, using standard properties of Gaussian complexity
\cite{LedouxTalagrand1991}:
\begin{align*}
\mathfrak{G}_T^\text{seq}(\cF, \bz)
\leq
3 \sqrt{\frac{\pi}{2} \log T}
\E_{\be, \bs}\Bigg[
\sup_{f \in \cF} \Big( \sum_{t=1}^T q_t \s_t \e_t f(z_t(\bs)) \Big)\Bigg],
\end{align*}
where $\be$ is an independent sequence of Rademacher variables.
We re-arrange $\bz$ into $\bz^{\be}$ so that $z_t(\bs) = z^{\be}_t(\be\bs)$
for all $\bs \in \set{\pm 1}^T$ and it follows that
\begin{align*}
\E_{\be, \bs}\Bigg[
\sup_{f \in \cF} \Big( \sum_{t=1}^T q_t \s_t \e_t f(z_t(\bs)) \Big)\Bigg]
& = \E_{\be, \bs}\Bigg[
\sup_{f \in \cF} \Big( \sum_{t=1}^T q_t \s_t \e_t f(z^{\be}_t(\be\bs)) \Big)\Bigg]
\\ &\leq \sup_{\bz}
\E_{\be, \bs}\Bigg[
\sup_{f \in \cF} \Big( \sum_{t=1}^T q_t \s_t \e_t f(z_t(\be\bs)) \Big)\Bigg] \\
&=
\E_{\bs}\Bigg[
\sup_{f \in \cF} \Big( \sum_{t=1}^T q_t \s_t f(z_t(\bs)) \Big)\Bigg].
\end{align*}
Therefore, the following inequality holds
\begin{align*}
\sup_{\alpha > 0} \frac{\alpha}{2} \sqrt{\log \cN_2(2\alpha, \cF, \bz)}
\leq 3 \sqrt{\frac{\pi}{2} \log T}
\sR_T(\cF)
\end{align*}
and conclusion of this theorem follows by taking supremum with
respect to $\bz$ on both sides of this inequality.
\end{proof}

Observe that since 
$\E_{\bv \sim T(\P)}[ \cN_1(\alpha, \cG, \bv) ]
\leq \E_{\bv \sim T(\P)}[ \cN_2(\alpha, \cG, \bv) ] \leq
\cN_2(\alpha, \cG)$. Then setting $\alpha = \|\bq\|_2 / 2$,
applying Corollary~\ref{cor:bound}
and Theorem~\ref{th:sudakov},
and using the fact that $\sqrt{x+y} \leq \sqrt{x} + \sqrt{y}$
for $x,y>0$, yields the following result.

\begin{corollary}
\label{cor:rad-bound}
For any $\delta > 0$, with probability at least $1 - \delta$, for all
$f \in \cF$ and all $\alpha>0$,
\begin{align*}
\E[f(Z_{T+1})|\bZ_1^T] \leq \sum_{t = 1}^T q_t f(Z_{t}) + \disc(\bq)
+ \|\bq\|_2 +
6 M  \sqrt{\pi \log T} \sR_T(\cF) + M \|\bq\|_2 \sqrt{2 \log \frac{1}{\d}}.
\end{align*}
\end{corollary}

As we have already mentioned in Section~\ref{sec:prelim},
sequential Rademacher complexity can be further upper bounded
in terms of sequential metric entropy, sequential Littlestone dimension,
maximal sequential covering numbers and other combinatorial
notions of sequential complexity of $\cF$. These notions have
been extensively studied in the past \cite{RakhlinSridharanTewari2015b}.

\ignore{We will use Corollary~\ref{cor:rad-bound} in Section~\ref{sec:kernel-classes}
to study generalization properties of kernel-based hypothesis classes.}

We conclude this section by observing that our results also
hold in the case when $q_t = q_t(f, X_{T+1}, Z_t)$, which is
a common heuristic used in some algorithms for
forecasting non-stationary time series \citep{Lorenz1969,ZhaoGiannakis2016}.
We formalize this result in the following theorem.
\begin{theorem}
\label{cor:bound-2}
Let $q \colon \cF \times \cX \times \cZ \mapsto [-B,B]$ and suppose
$X_{T+1}$ is $\bZ_1^T$-measurable. Then,
for any $\delta > 0$, with probability at least $1 - \delta$, for all
$f \in \cF$ and all $\alpha>0$,
\begin{align*}
\E[f(Z_{T+1})|\bZ_1^T] \leq \frac{1}{T}\sum_{t = 1}^T q(f, X_{T+1}, Z_t)
f(Z_{t}) + \disc(q) + 2 \alpha + 2 MB
\sqrt{2\frac{\log\frac{\E_{\bv \sim T(\P)}[ \cN_1(\alpha, \cG, \bv) ]}{\delta}}{T}},
\end{align*}
where $\disc(q)$ is defined by
\begin{align}
\disc(q) = \sup_{f \in \cF} \bigg( \E [f(Z_{T + 1})|\bZ_1^T] -
\sum_{t = 1}^T \E_{Z_t} [q(f, X_{T+1}, Z_{t}) f(Z_{t}) | \bZ_1^{t-1}] \bigg),
\end{align}
where $\cG = \set{(x,z) \mapsto q(f, x, z) f(z) \colon f \in \cF}$.
\end{theorem} 

We illustrate this result with some examples. Consider for instance
a Gaussian Markov process with $\bP_t(\cdot|\bZ_1^T)$ being
a normal distribution with mean $Z_{t-1}$ and unit variance.
Suppose $f(h,z) = \ell(\|h(x) - y\|_2)$ for some function $\ell$.
We let
$q(h, x', (x,y)) = \exp(-\frac{1}{2}\|y-h(x)-x'+h(x')\|_2^2)/
\exp\Big(-\frac{1}{2}\|x-y\|_2^2\Big)$ and
observe that for any $f$:
\begin{align*}
\E_{Z_t} [q(X_{T+1}, Z_{t}) f(Z_{t}) | \bZ_1^{t-1}]
&= \int \ell(\|y-h(X_t)\|_2) q(h, X_{T+1}, (X_{t},y))\exp
\Big(-\frac{1}{2}\|y - X_{t}\|_2^2\Big)dy \\
&= \int \ell(\|y-h(X_t)\|_2) \exp\Big(-\frac{1}{2}\|y-h(X_t) - X_{T+1}
+ h(X_{T+1})\|_2^2\Big) dy \\
&= \int \ell(\|x\|_2)
\exp\Big(-\frac{1}{2}\|x - X_{T+1} + h(X_{T+1})\|_2^2\Big) dx \\
&= \int \ell(\|y-h(X_{T+1})\|_2)
\exp\Big(-\frac{1}{2}\|y - X_{T+1}\|_2^2\Big) dy \\
&= \E[f(Z_{T+1})|\bZ_1^T],
\end{align*}
which show that $\disc(q) = 0$ in this case. More generally,
if $\bZ$ is time-homogeneous Markov process then one can use
Radon-Nikodym derivative $\frac{d \bP(\cdot|x') }{ d \bP(\cdot|x)}(y - h(x)+h(x'))$ for $q$, which will again lead
to zero discrepancy. The major obstacle for this approach
is that Radon-Nikodym derivatives are typically unknown and
one needs to learn them from data via density estimation, which itself
can be a difficult task. In Section~\ref{sec:kernel-classes},
we investigate an alternative approach
to choosing weights $\bq$ based on extending results in
Theorem~\ref{th:bound} to hold uniformly over weight vectors $\bq$.

\section{Kernel-Based Hypotheses with Regression Losses}
\label{sec:kernel-classes}

\ignore{One of the main technical tools used in our analysis is the notion
of \emph{sequential Rademacher complexity}.
Let $\cG$ be a set of functions from $\cZ$ to $\Rset$.

 This a combinatorial measure of complexity which
makes bounds based on this notion coarser than those of
Theorem~\ref{th:bound}, which are stated in terms of expected covering
numbers. However, it turns out that this coarser analysis is
sufficient for the derivation of our algorithms in
Section~\ref{sec:algo}. We also remark that most of the results in
this section can be tightened using the notion of
\emph{distribution-dependent} Rademacher complexity, but we defer
these results to the full version of the paper.}

In this section, we present generalization bounds for kernel-based
hypothesis with regression losses. Our results in this section,
are based on the learning guarantee presented in Corollary~\ref{cor:rad-bound}
in terms of sequential Rademacher complexity.
Our first result is a bound on the sequential Rademacher complexity
of the kernel-based hypothesis with regression losses.

\begin{lemma}
\label{lm:rademacher-bound}
Let $p \geq 1$ and
$\cF = \set{(\bx, y) \to (\bw \cdot \Psi(\bx) - y)^p \colon
  \|\bw\|_\cH \leq \Lambda}$
where $\cH$ is a Hilbert space and $\Psi\colon \cX \to \cH$ a feature
map.  Assume that the condition $|\bw \cdot \bx - y| \leq M$ holds for
all $(\bx, y)\in \cZ$ and all $\bw$ such that
$\|\bw\|_\cH \leq \Lambda$.  Then, the following inequalities hold:
\begin{align}
\sR_T(\cF) 
\leq p M^{p-1} C_T \sR_T(H) 
\leq C_T \Big(pM^{p-1} \Lambda r \|\bq\|_2 \Big),
\end{align}
where 
$K$ is a PDS kernel associated to $\cH$,
$H = \set{\bx \to \bw \cdot \Psi(\bx) : \|\bw\|_\cH \leq \Lambda}$,
$r = \sup_{x} K(x, x)$, and
$C_T = 8 (1 + 4\sqrt{2} \log^{3/2} (eT^2) )$.
\end{lemma}

\begin{proof}
We begin the proof by setting
$q_t f(\bz_t(\bs)) = q_t (\bw \cdot \Psi(\bx_t(\sigma)) - \by_t(\sigma))^2
= \tfrac{1}{T}(\bw \cdot \bx'(\bs) - \by'_t)^2$, where
$\bx'_t(\bs) = \sqrt{T q_t} \Psi(\bx_t(\sigma))$
and $\by'_t(\bs) = \sqrt{T q_t} \by_t(\bs)$.
We let $\bz'_t = (\bx'_t, \by'_t)$. Then we observe that
\begin{align*}
\sR_T(\cF) &= \sup_{\bz' = (\bx', \by')} \E_{\bs}
\Bigg[\sup_{\bw} \frac{1}{T} \sum_{t = 1}^T \s_t
(\bw \cdot \bx'_t(\bs) - \by'_t(\bs))^p \Bigg]
\\ &=  \sup_{\bz = (\bx,\by)} \E_{\bs}
\Bigg[\sup_{\bw} \sum_{t = 1}^T q_t \s_t
(\bw \cdot \bx_t(\bs) - \by_t(\bs))^p \Bigg].
\end{align*}
Since $x \to |x|^p$ is $pM^{p-1}$-Lipschitz over $[-M,M]$, by Lemma~13
in \citep{RakhlinSridharanTewari2015}, the following bound holds:
\begin{align*}
\sR_T(\cF) \leq p M^{p-1} C_T \sR_T(H'),
\end{align*}
where $H' = \set{(\bx, y) \to \bw \cdot \Psi(\bx) - y :
 \|\bw\|_\cH \leq \Lambda}$. Note that Lemma~13 requires that
$\sR_T(H') > 1/T$ which is guaranteed by Khintchine's inequality.
By definition of the sequential Rademacher
complexity
\begin{align*}
\sR_T(H') 
& = \sup_{(\bx, y)}
\E_{\bs} \Bigg[ \sup_{\bw}  \sum_{t = 1}^T \s_t q_t
  (\bw \cdot \Psi(\bx_t(\bs)) - y(\bs))   \Bigg] \\
& = \sup_{\bx}
\E_{\bs} \Bigg[ \sup_{\bw}  \sum_{t = 1}^T \s_t q_t
  \bw \cdot \Psi(\bx_t(\bs)) \Bigg]
+ \sup_y \E_{\bs}  \Bigg[ \sum_{t = 1}^T\s_t q_t y(\bs)  \Bigg] =  \sR_T(H),
\end{align*}
where for the last equality we used the fact that $\s_t$s are mean zero
random variables and $\s_t$ is independent
of $y(\bs) = y(\s_1, \s_2, \ldots, \s_{t-1})$.
This proves the first result. To prove the second bound
we observe that the right-hand side can be bounded as follows:
\begin{align*}
 \sup_{\bx}
\E_{\bs} \Bigg[ \sup_{\bw}  \sum_{t = 1}^T \s_t q_t
  \bw \cdot \Psi(\bx_t(\bs)) \Bigg] &\leq
\Lambda
 \sup_{\bx} \E_{\bs} \Bigg\| \sum_{t = 1}^T \s_t q_t
  \Psi(\bx_t(\bs)) \Bigg\|_\cH \\
&\leq \Lambda
 \sup_{\bx} \sqrt{\E_{\bs} \Bigg\| \sum_{t = 1}^T \s_t q_t
  \Psi(\bx_t(\bs)) \Bigg\|^2_\cH} \\
 &= \Lambda
 \sup_{\bx} \sqrt{\E_{\bs} \Bigg[\sum_{t,s=1}^T \s_t \s_s q_t q_s
    \Psi(\bx_t(\bs)) \cdot \Psi(\bx_s(\bs)) \Bigg]} \\
 &\leq  \Lambda
 \sup_{\bx} \sqrt{ \sum_{t = 1}^T q^2_t \E_{\bs}[K(x_t(\bs), x_t(\bs))] } \\
 &\leq \Lambda r \|\bq\|_2,
\end{align*}
where again we are using the fact that if $s < t$ then
\begin{align*}
\E_{\bs}[\s_t \s_s q_t q_s K(x_t(\s), x_s(\s))] =
\E_{\bs}[\s_t] \E_{\bs} [\s_s q_t q_s K(x_t(\s), x_s(\s))] = 0
\end{align*}
by the independence of $\s_t$ from $\s_s$, $x_t(\s) = x_t(\s_1, \ldots, \s_{t-1})$ and $x_s(\s) = x_s(\s_1, \ldots, \s_s)$.
\end{proof}

Our next result establishes a high-probability learning guarantee
for kernel-based hypothesis. Combining Corollary~\ref{cor:rad-bound}
with Lemma~\ref{lm:rademacher-bound}, yields the following result.

\begin{theorem}
\label{th:lin-hypothesis-bound}
Let $p \geq 1$ and
$\cF = \set{(\bx, y) \to (\bw \cdot \Psi(\bx) - y)^p \colon
  \|\bw\|_\cH \leq \Lambda}$
where $\cH$ is a Hilbert space and $\Psi\colon \cX \to \cH$ a feature
map.  Assume that the condition $|\bw \cdot \bx - y| \leq M$ holds for
all $(\bx, y)\in \cZ$ and all $\bw$ such that
$\|\bw\|_\cH \leq \Lambda$.  If $\bZ_1^T = (\bX_1^T,\bY_1^T)$ is a
sequence of random variables then, for any $\d >0$, with probability
at least $1 - \d$ the following holds for all
$h \in \set{\bx \to \bw \cdot \Psi(\bx) \colon \|\bw\|_{\cH} \leq
  \Lambda}$:
\begin{align*}
\E[(h(X_{T+1}) - Y_{T+1})^p|\bZ_1^{T}]
\leq \sum_{t = 1}^T q_t (h(X_{t}) - Y_{t})^p &+ \disc(\bq)
+ C_T  \Lambda r \|\bq\|_2 + 2 M \|q\|_2 \sqrt{2 \log \frac{1}{\d}},
\end{align*}
where
$C_T = 48 pM^p \sqrt{\pi \log T} (1 + 4\sqrt{2} \log^{3/2} (eT^2))$.
Thus, for $p = 2$,
\begin{align*}
\E[(h(X_{T+1}) - Y_{T+1})^2|\bZ_1^{T}] 
\leq &\sum_{t = 1}^T q_t (h(X_{t}) - Y_{t})^2 + \disc(\bq) +
O\Bigg( (\log^2 T) \Lambda r \|\bq\|_2 \Bigg).
\end{align*}
\end{theorem}

The results in Theorem~\ref{th:lin-hypothesis-bound} (as well as
Theorem~\ref{th:bound}) can be extended
to hold uniformly over $\bq$ and we provide exact statement
in Theorem~\ref{th:uniform} in Appendix~\ref{sec:proofs}.
This result suggests that we should try to minimize
$\sum_{t = 1}^T q_t f(Z_t) + \disc(\bq)$ over $\bq$
and $\bw$. This insight is used to develop our algorithmic solutions
for forecasting non-stationary time series in Section~\ref{sec:algo}.

\section{Stability Analysis}
\label{sec:stability}

In this section, we study the problem of time series forecasting
through the lens of algorithmic stability.
As in the classical learning theory, algorithmic stability
provides an alternative tool for deriving generalizations
bounds for a class of stable algorithms.
Let $h(\bz)(\cdot)$ denotes a model obtained from training $\cA$
on the sample $\bz$.
We say that an algorithm $\cA$ is uniformly $\beta$-stable if
for any $z = (x,y) \in \cZ$ and
for any two samples $\bz$ and $\bz'$ that differ by exactly one point,
the following inequality holds:
\begin{align*}
|L(h(\bz), z) - L(h(\bz'), z)| \leq \beta.
\end{align*}
The following result can then be shown.

\begin{theorem}
\label{th:stability-bound}
Let $\cA$ be a $\beta$-stable learning algorithm and
let $\bZ_1^T$ be any sequence of random variables.
Let $\bq = (q_1, \ldots, q_T)$ be any weight vector. For
any $\d > 0$, each of the following bounds holds with probability at
least $1 - \d$:
\begin{align*}
&\cL_{T+1}(h(\bZ_1^T), \bZ_1^T)
 \leq
\sum_{t = 1}^T q_t L(h(\bZ_1^T), Z_{t+1}) + \|\bq\|_1  \beta + \disc(\bq)
 + 2 M \|\bq\|_2 \sqrt{2 \log \frac{2}{\d}}, \\
&\sum_{t = 1}^T q_t L(h(\bZ_1^T), Z_{t+1})
 \leq
 \cL_{T+1}(h(\bZ_1^T), \bZ_1^T) +  \|\bq\|_1 \beta  + \disc(\bq)
 + 2 M \|\bq\|_2 \sqrt{2 \log \frac{2}{\d}}. 
\end{align*}
\end{theorem}

\begin{proof}
For each, we let $\bZ_t^T$ and $\tl \bZ_t^T$ be independent
sequences of random variables drawn from $\bP_{t}^T(\cdot | \bZ_1^{t-1})$.
Define $\hh \bZ(t)$ as the sequence
$(Z_1, \ldots, Z_t, \tl Z_{t+1}, \ldots, \tl Z_T)$ and
observe for any $g$
and any $s \leq t$ the following holds:
\begin{align}
\label{eq:expectation-equivalence}
\E[g(\bZ_1^T) | \bZ_1^s] = \E[g(\hh \bZ(t)) | \bZ_1^s].
\end{align}

Consider
$A_t = \cL_{t+1}(h_t(\bZ_1^T), \bZ_1^t)-\cL_{t+1}(h_t(\hh \bZ(t)), \bZ_1^t)$
and observe that this process forms a martingale difference sequence.
Indeed,
\begin{align*}
&\E_{\bZ_t^T, \tl \bZ_t^T}\Big[
\cL_{t+1}(h_t(\bZ_1^T), \bZ_1^t)-\cL_{t+1}(h_t(\hh \bZ(t)), \bZ_1^t)\Big] \\
&=
\E_{\bZ_t^T, \tl \bZ_t^T} \Big[
\E_{Z}[L(h(\bZ_1^T), Z)| \bZ_1^t] -
\E_{Z}[L(h(\hh \bZ(t)), Z)| \bZ_1^t] \Big]\\
&= \E_{\bZ_t^T, Z} [L(h(\bZ_1^T), Z)| \bZ_1^t] -
\E_{\hh \bZ(t), Z}[L(h(\hh \bZ(t)), Z)| \bZ_1^t] = 0,
\end{align*}
where $Z \sim \bP(\cdot|\bZ_1^t)$ is independent of
$\bZ_t^T$ and $\tl \bZ_t^T$ and where the last equality follows from
\eqref{eq:expectation-equivalence}.
Therefore, by Azuma's inequality, for any $\d > 0$ with probability
at least $1-\d/2$,
$\sum_{t=1}^T q_t A_t \leq \|\bq\|_2 \sqrt{2 \log \frac{2}{\d}}$.

Similarly, observe that
$B_t =
\E_{Z_{t+1}} [L(h_t(\hh \bZ(t+1)), Z_{t+1}) | \bZ_1^t] -
L(h(\bZ_1^T), Z_{t+1})$ is also
a martingale difference sequence by the same argument as
above. An important technical detail
is that
$\E_{Z_{t+1}} [L(h_t(\hh \bZ(t+1)), Z_{t+1}) | \bZ_1^t] \neq
\E_{\tl Z_{t+1}} [L(h_t(\hh \bZ(t+1)), \tl Z_{t+1}) | \bZ_1^t] = 
\cL_{t+1}(h_t(\hh \bZ(t+1)), \bZ_1^t)$ and
$\cL_{t+1}(h_t(\hh \bZ(t+1)), \bZ_1^t) - L(h_t(\bZ_1^T), Z_{t+1})$
is not a martingale difference sequence.

By Azuma's inequality, 
for any $\d > 0$ with probability
at least $1-\d/2$,
$\sum_{t=1}^T q_t B_t \leq \|\bq\|_2 \sqrt{2 \log \frac{2}{\d}}$.
Since, by stability
\begin{align*}
&|\E_{Z_{t+1}} [L(h_t(\hh \bZ(t+1)), Z_{t+1}) | \bZ_1^t] - \cL_{t+1}(h(\hh \bZ(t)), \bZ_1^t)| \leq \beta
\end{align*}
it follows that, for any $\d>0$, with probability at least $1 - \d$,
\begin{align*}
\sum_{t=1}^T q_t\cL_{t+1}(h_t(\bZ_1^T), \bZ_1^t)
\leq \sum_{t=1}^T q_t L(h(\bZ_1^T), Z_{t+1})
+ \|\bq\|_1 \beta + 2 \|q\|_2 \sqrt{2 \log \frac{2}{\d}}.
\end{align*}
The first statement of the theorem follows from the fact that,
by definition of discrepancy,
\begin{align*}
\cL_{T+1}(h(\bZ_1^T), \bZ_1^T) \leq \sum_{t=1}^T q_t\cL_{t+1}(h_t(\bZ_1^T), \bZ_1^t) + \disc (\bq)
\end{align*}
and the second statement follows by symmetry.
\ignore{

Therefore,  forms a martingale sequence. Hence, with probability at least
$1-\delta/2$, $\sum_{t=1}^T q_t A_t \leq  M \|\bq\|_2 \sqrt{2 \log \frac{2}{\d}}$ by Azuma's inequality.

Similarly, $B_t = \E_{Z_{t+1}} [L(h_t(\bZ_1^t, Z_{t+1}, \tl \bZ_{t+2}^T), Z_{t+1}) | \bZ_1^t] -
L(h_t(\bZ_1^T), Z_{t+1})$ is also a martingale difference and
with probability at least
$1-\delta/2$, $\sum_{t=1}^T q_t B_t \leq  M \|\bq\|_2 \sqrt{2 \log \frac{2}{\d}}$.

Since $\cL_{t+1}(h_t(\hh \bZ(t)), \bZ_1^t) - \E_{Z_{t+1}} [L(h_t(\hh \bZ(t+1)), Z_{t+1}) | \bZ_1^t] \leq \beta_t$ by stability, it follows that for any $\d$
with probability at least $1-\d$ the following bound holds:
\begin{align*}
\sum_{t = 1}^T q_t \cL_{t+1}(h_t, \bZ_1^t)
& \leq
\sum_{t = 1}^T q_t L(h_t, Z_{t+1}) + \sum_{t = 1}^T q_t \beta_t +
 2 M \|\bq\|_2 \sqrt{2 \log \frac{2}{\d}},
\end{align*}
and the first statement of the lemma follows from the definition of the
discrepancy. The second statement follows by symmetry.}
\end{proof}

A large array of existing learning algorithms can be shown to be stable
\citep{BousquetElisseeff2002,MohriRostamizadeh2010}.
In particular, kernel-based regularization algorithms defined
by minimizing the following objective:
\begin{align}
\label{eq:kernel-based-reg-obj}
F_{\bZ_1^T}(h) = \frac{1}{T}\sum_{t=1}^T L(h, Z_t) + \lambda \|h\|^2_\cH
\end{align}
are known to be stable. In Equation~\eqref{eq:kernel-based-reg-obj},
$h$ belongs to the reproducing kernel Hilbert space $\cH$ associated to some
PDS kernel $K$, $\|\cdot\|_\cH$ is the norm defined by $K$ and $\lambda \geq 0$
is a hyperparameter. If the loss function is assumed to be convex
and $\s$-admissible, then the algorithm can be shown to admit a stability
coefficient bounded as follows $\beta \leq \frac{\s^2 r^2}{T \lambda}$,
where $r^2 = \sup_x K(x,x)$. Recall the loss function is
$\s$-admissible if for any $h,h' \in H$ and for all $z = (x,y) \in \cZ$,
$|L(h, z) - L(h', z)| \leq \s |h(x) - h'(x)|$. In particular,
this assumption holds if $L$ is $\s$-Lipchitz with respect to its first
argument. 

In our case, we are interested in a $\bq$-weighted version of the objective
in \eqref{eq:kernel-based-reg-obj}:
\begin{align}
\label{eq:q-kernel-based-reg-obj}
F_{\bZ_1^T}(h, \bq) = \sum_{t=1}^T q_t L(h, Z_t) + \lambda \|h\|^2_\cH.
\end{align}

We have the following result.

\begin{theorem}
\label{th:stability-coef}

Let $K$ be a positive definite symmetric kernel such that
$r^2 = \sup_x K(x,x)$ and let $L$ be a convex and $\s$-admissible loss
function. Let $\bq = (q_1, \ldots, q_T)$ be any
non-negative weight vector.
Then, the kernel-based regularization algorithm defined by the
minimization of $F_{\bZ_1^T}(h, \bq)$ in \eqref{eq:q-kernel-based-reg-obj}
is $\beta$-stable with $\beta \leq \frac{\s^2 r^2 \|\bq\|_\infty}{\lambda}$.
\end{theorem}

\begin{proof}
The proof of this result follows the same argument as in the case of
uniform weights \citep{MohriRostamizadehTalwalkar2012}.

Let $S = \bZ_1^T$ and $S'$ be a sample that difers from $S$ by exactly one
point, say $Z'_t$. Assume that $h$ and $h'$ are minimizers of
$F_S = F_S(\cdot, \bq)$ and $F_{S'} = F_{S'}(\cdot, \bq)$ respectively.
We let $B_{F_S}$ denote the generalized Bregman divergence defined
by $F_S$, that is,
\begin{align*}
B_{F_S}(h_1 \parallel h_2) = F_S(h_1) - F_S(h_2) -
\langle h_1 - h_2, \d F_S(h_2) \rangle_\cH,
\end{align*}
where $\d F_S(h)$ is denotes any element of subgradient of $F_S$ at $h$
such that the following holds:
$\d (\sum_{t=1}^T q_t L(h, Z_t)) = \d F_S(h) - \lambda \nabla \|h\|_\cH^2$
and $\d F_s(h) = 0$ whenever $h$ is a minimizer of $F_S$. Note
that this implies that $B_{F_S} = B_{\hh R_S} + \lambda B_N$,
where $N(h) = \|h\|_\cH^2$ and $\hh R_S = \sum_{t=1}^T q_t L(h, Z_t)$.
Then, since the generalized Bregman divergence is non-negative,  we can write,
\begin{align*}
B_{F_S}(h' \parallel h) + B_{F_{S'}}(h \parallel h')
\geq \lambda
B_{N}(h' \parallel h) + B_{F_{N}}(h \parallel h').
\end{align*}
Observe that $B_N(h'\parallel h) + B_N(h\parallel h')
= - \langle h'-h, 2h \rangle - \langle h - h', 2h' \rangle = 2\|h' - h\|_\cH^2$. Let $\Delta h = h' - h$. Then it follows that
\begin{align*}
2 \lambda \|\Delta h \|_\cH^2 &\leq
B_{F_S}(h' \parallel h) + B_{F_{S'}}(h \parallel h') \\
=& F_S(h') - F_S(h) -
\langle h' - h, \d F_{S}(h) \rangle_\cH
\\ &+ F_{S'}(h) - F_{S'}(h') -
\langle h - h, \d F_{S'}(h') \rangle_\cH \\
=& F_S(h') - F_S(h) + F_{S'}(h) - F_{S'}(h') \\
=& \hh R_S(h') - \hh R_S(h) + \hh R_{S'}(h) - \hh R_{S'}(h'),
\end{align*}
where the first equality uses the definition of the generalized Bregman
divergence, second equality follows from the fact that $h$ and $h'$
are minimizers and last equality follows from the definition of
$F_S$ and $F_{S'}$. By $\s$-admissibility of $L$ and the fact that
$S$ and $S'$ differ by exactly one point it follows that
\begin{align*}
2 \lambda \|\Delta h \|_\cH^2 &\leq
q_t ( L(h', Z_t) - L(h, Z_t) + L(h, Z'_t) - L(h', Z'_t) )
\\ &\leq \s q_t (|\Delta h(X_t)| + |\Delta h(X'_t)|).  
\end{align*}
By the reproducing kernel property and the Cauchy-Schwartz inequality,
for all $x \in \cX$,
\begin{align*}
\Delta(x) = \langle \Delta h, K(x, \cdot) \rangle_\cH
\leq \|\Delta h\|_\cH \|K(x,\cdot)\|_\cH \leq r \|\Delta h\|_\cH.
\end{align*}
It follows that $\|\Delta h\|_\cH \leq \frac{\s r \| \bq \|_\infty}{\lambda}$.
Therefore, by $\s$-admissibility and the reproducing kernel property
\begin{align*}
|L(h', z) - L(h, z)| \leq \s |\Delta h(x)| \leq r \s \|\Delta\|_\cH
\leq \frac{\s^2 r^2 \| \bq \|_\infty}{\lambda}
\end{align*}
for all $z = (x,y)$, which concludes the proof.
\end{proof}

This result combined with Theorem~\ref{th:stability-bound} immediately
provides learning guarantees for $\bq$-weighted analogues of
support vector regression (SVR) and kernel ridge regression (KRR) algorithms
use $L_\e(y,y') = (|y-y'| -\e)_+$ and $L_2 (y,y') = (y-y')^2$
as loss functions.
These loss functions are $\sigma$-admissible with $\sigma=1$ and
$\sigma=2\sqrt{M}$ respectively, where $M$ is a bound on a loss function. 

\begin{corollary}
\label{cor:stability-krr-svr}
Assume that $r^2 = \sup_{x \in \cX} K(x,x)$. Suppose $L_\e \leq M_1$
and $L_2 \leq M_2$. Let $h_{\bZ_1^T}^\text{SVR}$ and $h_{\bZ_1^T}^\text{KRR}$
denote the hypothesis returned by SVR and KRR respectively when
trained on sample $\bZ_1^T$. Let $\bq$ be any vector in the probability
simplex. Then, for any $\d > 0$, each of the following bounds holds,
with probability at least $1-\d$:
\begin{align*}
&\cL_{T+1}(h_{\bZ_1^T}^\text{SVR}, \bZ_1^T)
 \leq
\sum_{t = 1}^T q_t L(h_{\bZ_1^T}^\text{SVR}, Z_{t+1}) +
 \frac{ r^2 \|\bq\|_\infty }{\lambda} + \disc(\bq)
 + 2 M \|\bq\|_2 \sqrt{2 \log \frac{2}{\d}}, \\
&\cL_{T+1}(h_{\bZ_1^T}^\text{KRR}, \bZ_1^T)
 \!\leq\!
\sum_{t = 1}^T q_t L(h_{\bZ_1^T}^\text{KRR}, Z_{t+1})\! +\! \frac{ 4 M r^2 \|\bq\|_\infty }{\lambda} + \disc(\bq)
 \!+\! 2 M \|\bq\|_2 \sqrt{2 \log \frac{2}{\d}}.
\end{align*}

\end{corollary}

\section{Estimating Discrepancy}
\label{sec:discrepancy}

In Section~\ref{sec:bounds}, we showed that the discrepancy $\disc(\bq)$
is crucial for forecasting non-stationary time series. In particular,
if we could select a distribution $\bq$ over the sample $\bZ_1^T$ that
would minimize the discrepancy $\disc(\bq)$ and use it to weight training
points, then we would have a better learning guarantee for an
algorithm trained on this weighted sample. In some special cases, the
discrepancy $\disc(\bq)$ can be computed analytically. However, in
general, we do not have access to the distribution of $\bZ_1^T$ and
hence we need to estimate the discrepancy from data. Furthermore,
in practice, we never observe $Z_{T + 1}$ and it is not possible to
estimate $\disc(\bq)$ without some further assumptions. One natural
assumption is that the distribution $\bP_t$ of $Z_t$ does not change
drastically with $t$ on average. Under this assumption the last $s$
observations $\bZ_{T-s+1}^T$ are effectively drawn from the
distribution close to $\bP_{T + 1}$. More precisely, we can write
\begin{align*}
\disc(\bq) \leq &
\sup_{f \in \cF} \bigg(
\frac{1}{s}\sum_{t = T-s+1}^T \E[f(Z_t) | \bZ_1^{t-1}] -
                        \sum_{t = 1}^T q_t \E[f(Z_t) | \bZ_1^{t-1}] \bigg) \\
& + \sup_{f \in \cF} \bigg(\E[f(Z_{T + 1}) | \bZ_1^{T}] -
              \frac{1}{s}\sum_{t = T-s+1}^T \E[f(Z_t) | \bZ_1^{t-1}] \bigg).
\end{align*}
We will assume that the second term, denoted by $\disc_s$, is
sufficiently small and will show that the first term can be estimated
from data. But, we first note that our assumption is necessary for
learning in this setting. Observe that
\begin{align*}
\sup_{f \in \cF} \Big(\E[Z_{T + 1}|\bZ_1^T] - \E[f(Z_r)|\bZ_1^{r-1}]\Big)
&\leq
\sum_{t=r}^T \sup_{f \in \cF} \Big( \E[f(Z_{t + 1}) | \bZ_1^{t}] -
\E[f(Z_t)|\bZ_1^{t-1}] \Big)\\ &\leq
M \sum_{t=r}^T \| \bP_{t + 1}(\cdot|\bZ_1^{t}) -
                  \bP_{t}(\cdot|\bZ_1^{t-1}) \|_\text{TV},
\end{align*}
for all $r = T - s + 1, \ldots, T$.  Therefore, we must have
\begin{align*}
\disc_s \leq \frac{1}{s} \sum_{t=T-s+1}
\sup_{f \in \cF} \Big(\E[Z_{T + 1}|\bZ_1^{T}] - \E[f(Z_t)|\bZ_1^t]\Big) \leq
\frac{s+1}{2} M \gamma,
\end{align*}
where
$\gamma \!=\! \sup_t \! \| \bP_{t + 1}(\cdot|\bZ_1^{t}) -
\bP_t(\cdot|\bZ_1^{t-1}) \|_\text{TV}$.
\cite{BarveLong1996} showed that $[\VCdim(H) \gamma]^{\frac{1}{3}}$ is a lower
bound on the generalization error in the setting of binary
classification where $\bZ_1^T$ is a sequence of independent but not
identically distributed random variables (drifting). This setting is a
special case of the more general scenario that we are considering.

The following result shows that we can estimate the first term
in the upper bound on $\disc(\bq)$.

\begin{theorem}
\label{th:discrepancy}
Let $\bZ_1^T$ be a sequence of random variables.  Then, for any
$\d > 0$, with probability at
least $1 - \delta$, the following holds for all $\alpha > 0$:
\begin{align*}
\sup_{f \in \cF} \Bigg( \sum_{t = 1}^T(p_t - q_t)
 \E[f(Z_t)|\bZ_1^{t-1}]\Bigg) \leq
 & \sup_{f \in \cF} \Bigg( \sum_{t = 1}^T (p_t - q_t) f(Z_t) \Bigg) +  B,
\end{align*}
where
$B = 2\alpha + M \|\bq - \bp\|_2 \sqrt{2\log{\frac{\E_{\bz \sim T(\bp)}[
      \cN_1(\alpha, \cF, \bz) ]}{\d}}} $
and where $\bp$ is the uniform distribution over the last $s$ points.
\end{theorem}

\begin{proof}
The first step is to observe that
\begin{align*}
\sup_{f \in \cF} \Bigg( \sum_{t = 1}^T(p_t - q_t)
 \E[f(Z_t)|\bZ_1^{t-1}]\Bigg) -
 & \sup_{f \in \cF} \Bigg( \sum_{t = 1}^T (p_t - q_t) f(Z_t) \Bigg)\\ &\leq
\sup_{f \in \cF} \Bigg( \sum_{t = 1}^T(p_t - q_t)(
 \E[f(Z_t)|\bZ_1^{t-1}] -  f(Z_t)) \Bigg)
\end{align*}
and then apply arguments similar to those in proof of Theorem~\ref{th:bound}.
\end{proof}

\ignore{The proof of this result is given in Appendix~\ref{sec:proofs}.}
Theorem~\ref{th:bound} and Theorem~\ref{th:discrepancy} combined with
the union bound yield the following result.
\begin{corollary}
\label{cor:data-dependent}
Let $\bZ_1^T$ be a sequence of random variables.  Then, for any
$\d > 0$, with probability at
least $1 - \delta$, the following holds for all $f \in \cF$ and all $\alpha>0$:
\begin{align*}
\E[f(Z_{T + 1})|\bZ_1^T] \leq
\sum_{t = 1}^T q_t f(Z_t) + \widetilde{\disc}(\bq) &+ \disc_s +
4\alpha \\ &+ M \big[ \| \bq \|_2 
+  \| \bq - \bp \|_2 \big] \sqrt{2\log{\tfrac{2\E_{\bv \sim T(\bp)}[ \cN_1(\alpha, \cG, \bz) ]}{\d}}},
\end{align*}
where 
$\widetilde{\disc}(\bq) =
\sup_{f \in \cF} \Big( \sum_{t = 1}^T (p_t - q_t) f(Z_t)\Big)$.
\end{corollary}

In Section~\ref{sec:algo}, we combine these results with
Theorem~\ref{th:uniform} that extends learning guarantees to hold
uniformly over $\bq$s to derive novel algorithms for non-stationary time series
prediction.

\section{Algorithms}
\label{sec:algo}

In this section, we use our learning guarantees to devise algorithms
for forecasting non-stationary time series.  We consider a broad
family of kernel-based hypothesis classes with regression losses
which we analyzed in
Section~\ref{sec:kernel-classes} and Section~\ref{sec:stability}.

Suppose $L$ is the squared loss and
$H = \set{\bx \to \bw \cdot \Psi(\bx) \colon \|\bw\|_\cH \leq
  \Lambda}$,
where $\Psi\colon \cX \to \cH$ is a feature mapping from $\cX$ to a
Hilbert space $\cH$. Theorem~\ref{th:lin-hypothesis-bound},
Theorem~\ref{th:uniform} and Theorem~\ref{cor:data-dependent}
suggest that we should solve
the following optimization problem:
\begin{align}
\label{eq:start-opt}
\min_{\bq \in \cQ, \bw} \Bigg\{\sum_{t=1}^T q_t (\bw \cdot \Psi(x_t) - y_t)^2
+ \widetilde{\disc}(\bq) + \lambda_1 \|\bw\|^2_\cH \Bigg\}
\end{align}
where $\lambda_1$ is a regularization hyperparameter and
$\cQ$ is some convex bounded domain. This optimization problem
is quadratic (and hence convex) in $\bw$ and convex in $\bq$
since $\widetilde{\disc}(\bq)$ is a convex function of $\bq$
as a supremum of linear functions. However, this objective
is not jointly convex in $(\bq, \bw)$. Of course, one could
apply a block coordinate descent to solve this objective
which alternates between optimizing over $\bq$ and solving a QP
in $\bw$. In general, no convergence guarantees can be provided
for this algorithm. In addition, each $\bq$-step involves
finding $\widetilde{\disc}(\bq)$ which in itself may be a costly
procedure. In the sequel we address both of these concerns.

First, let us observe that if
$a(\bw) = \sum_{t=1}^T p_t (\bw \cdot \Psi(x_t) - y_t)^2$ with $p_t$ being
a uniform distribution on the last $s$ points, then by definition
of empirical discrepancy
\begin{align*}
\widetilde{\disc}(\bw) &=
\sup_{\bw \leq \Lambda} \Bigg(a(\bw) - \sum_{t=1}^T q_t (\bw \cdot \Psi(x_t) - y_t)^2\Bigg)
\\ &=
\sup_{\bw \leq \Lambda} \Bigg(\sum_{t=1}^T(u_t - q_t)a(\bw) - \sum_{t=1}^T q_t ((\bw \cdot \Psi(x_t) - y_t)^2 - a(\bw))\Bigg)
\\ &\leq
\sum_{t=1}^T q_t d_t + \lambda_2 \|\bq - \bv\|_p 
\end{align*}
where $\lambda_2$ is some constant (a hyperparameter), $\bv$ is a prior
typically chosen to uniform weights $\bu$, $p\geq 1$ and $d_t$s are
instantaneous discrepancies defined by
\begin{align*}
d_t = \sup_{\bw \leq \Lambda} \Bigg| a(\bw) - (\bw \cdot \Psi(x_t) - y_t)^2\Bigg|.
\end{align*}
Note that $d_t$ can also be defined in terms of windows averages, where
$(\bw \cdot \Psi(x_t) - y_t)^2$ is replaced with $\frac{1}{2l}\sum_{s=t-l}^{t+l}(\bw \cdot \Psi(x_s) - y_s)^2$ for some $l$.

This leads to the following optimization problem:
\begin{align}
\label{eq:joint-opt}
\min_{\bq \in \cQ, \bw} & \ \bigg\{
   \sum_{t = 1}^T q_t (\bw \cdot \Psi(x_t) - y_t)^2
 +  \sum_{t = 1}^T d_t q_t  + \lambda_1 \|\bw\|_\cH^2
 + \lambda_2 \|\bq - \bv\|_p \bigg\}.
\end{align}
This optimization problem is still not convex but now $d_t$s
can be precomputed before solving \eqref{eq:joint-opt} which may
be considerably more efficient. For instance, the $\bq$-step
in the block coordinate descent reduces to a simple LP.
Below we show how \eqref{eq:joint-opt} can be
turned into a convex optimization problem when $\cQ = [0,1]^T$.

\subsection{Convex Optimization over $[0,1]^T$ and Dual Problems}

In this section, we consider the case when $\cQ = [0,1]^T$ and we show
how \eqref{eq:joint-opt} can be turned into a convex optimization
problem which then can be solved efficiently.
We apply change of 
variables $r_t = 1/q_t$, which leads to the following optimization problem:
\begin{align}
\label{eq:var-change-opt}
\min_{\br \in \cD, \bw} & \ \bigg\{
   \sum_{t = 1}^T \frac{(\bw \cdot \Psi(x_t) - y_t)^2 + d_t}{r_t}
 + \lambda_1 \|\bw\|_\cH^2
 + \lambda_2 \Bigg(\sum_{t=1}^T |r^{-1}_t - v_t|^p\Bigg)^{1/p} \bigg\},
\end{align}
where $\cD = \set{\br \colon r_t \geq 1}$. We can remove
the $(\cdot)^{1/p}$ on the last term by first turning it into
a constraint, raising it to the $p$th power and then moving it back
to the the objective by introducing a Lagrange multiplier:
\begin{align}
\label{eq:var-change-opt-2}
\min_{\br \in \cD, \bw} & \ \bigg\{
   \sum_{t = 1}^T \frac{(\bw \cdot \Psi(x_t) - y_t)^2 + d_t}{r_t}
 + \lambda_1 \|\bw\|_\cH^2
 + \lambda_2 \sum_{t=1}^T |r^{-1}_t - v_t|^p \bigg\}.
\end{align}
Note that the first
two terms in \eqref{eq:var-change-opt} are jointly convex in $(\br, \bw)$:
the first term is a sum of quadratic-over-linear functions which
is convex and the second term is a squared norm which is again convex.

The last step is to observe that
$|r_t^{-1} - v_t| = |v_t r_t^{-1} |^p |r_t - v_t^{-1}|^p \leq v_t^p |r_t - v_t|^p $ since $r_t^{-1} \leq 1$. Therefore, we have the following
optimization problem:
\begin{align}
\label{eq:convex-opt}
\min_{\br \in \cD, \bw} & \ \bigg\{
   \sum_{t = 1}^T \frac{(\bw \cdot \Psi(x_t) - y_t)^2 + d_t}{r_t}
 + \lambda_1 \|\bw\|_\cH^2
 + \lambda_2 \sum_{t=1}^T v_t^p |r_t - v_t^{-1}|^p \bigg\},
\end{align}
which is jointly convex over $(\br, \bw)$.

For many real-world problems, $\Psi$ is specified implicitly
via a PDS kernel $K$ and it is computationally advantageous to consider
the dual formulation of \eqref{eq:convex-opt}.
Using the dual problem associated to $\bw$
\eqref{eq:convex-opt} can be
rewritten as follows:
\begin{equation}
\label{eq:convex-dual}
\min_{\br \in \cD}\bigg\{ \max_{\ba} \Big\{ -\lambda_1 \sum_{t=1}^T r_t \alpha_t^2
 - \ba^T  \bK \ba + 2 \lambda_1 \ba^T \bY  \Big\} 
  + \sum_{t=1}^T \frac{d_t}{r_t} 
+ \lambda_2 \sum_{t=1}^T v_t^p |r_t - v_t^{-1}|^p \bigg\},
\end{equation}
where $\bK$ is the kernel matrix and $\bY = (y_1, \ldots, y_T)^T$.
We provide a full derivation of this result in Appendix~\ref{sec:opt}.

Both \eqref{eq:convex-opt} and \eqref{eq:convex-dual}
can be solved using standard descent methods,
where, at each iteration, we solve a standard QP in $\ba$ or $\bw$,
which admits a closed-form solution.

\subsection{Discrepancy Computation}

The final ingredient that is needed to solve optimization
problems \eqref{eq:convex-opt} or \eqref{eq:convex-dual} is
an algorithm to find instantaneous discrepancies $d_t$s.
Recall that in general these are defined as
\begin{align}
\label{eq:inst-disc}
\sup_{\bw' \leq \Lambda} \Bigg|\sum_{s=1}^T p_s \ell(\bw' \cdot
\Psi(x_s) - y_s) - \ell(\bw' \cdot \Psi(x_t) - y_t) \Bigg|
\end{align}
where $\ell$ is some specified loss function.
For an arbitrary $\ell$ this may be a difficult optimization problem.
However, observe that if $\ell$ is a convex function then
the objective in \eqref{eq:inst-disc} is a difference
of convex functions so we may use DC-programming to solve this problem.
For general loss functions,
the DC-programming approach only guarantees convergence to a stationary point.
However, for the squared loss, our problem can be cast as an instance of the
trust region problem, which can be solved globally using the DCA
algorithm of
\cite{TaoAn1998}.

\subsection{Two-Stage Algorithms}

An alternative simpler algorithm based on the data-dependent bounds of
Corollary~\ref{cor:data-dependent} consists of first finding a
distribution $\bq$ minimizing the discrepancy and
then using that to find a hypothesis minimizing the
(regularized) weighted empirical risk.  This leads to the following
two-stage procedure. First, we find a solution $\bq^*$ of the
following convex optimization problem:
\begin{equation}
\label{eq:stage-one}
\min_{\bq \geq 0} \bigg\{
 \sup_{\bw' \leq \Lambda} \Big(
    \sum_{t = 1}^T(p_t - q_t) (\bw' \cdot \Psi(x_t) - y_t)^2 \Big)
 \bigg\},
\end{equation}
where $\Lambda$ is parameter that can be selected via
cross-validation
(for example using techniques in \citep{KuznetsovMohri2016b}).
Our generalization bounds hold for arbitrary weights
$\bq$ but we restrict them to being positive sequences.  Note that other
regularization terms such as $\|\bq\|^2_2$ and $\|\bq - \bp\|^2_2$
from the bound of Corollary~\ref{cor:data-dependent} can be
incorporated in the optimization problem, but we discard them to
minimize the number of parameters. This problem can be solved using
standard descent optimization methods, where, at each step, we use
DC-programming to evaluate the supremum over $\bw'$.

The solution $\bq^*$ of \eqref{eq:stage-one} is then used
to solve the following (weighted) kernel ridge regression problem:
\begin{equation}
\label{eq:krr}
\min_\bw \bigg\{
 \sum_{t = 1}^T q^*_t (\bw \cdot \Psi(x_t) - y_t)^2 + \lambda_2 \| \bw \|^2_\cH
\bigg\}.
\end{equation}
Note that, in order to guarantee the convexity of this problem, we
require $\bq^* \geq 0$. This algorithm directly benefits from
the stability-based generalization bounds presented in
Corollary~\ref{cor:stability-krr-svr}. We note that one can
also use $\e$-sensitive loss instead of squared loss for this algorithm.

\ignore{
\subsection{Further Extensions}

We conclude this section with an observation that learning guarantees
that we presented in Section~\ref{sec:bounds}, Section~\ref{sec:kernel-classes}
and Section~\ref{sec:stability} can be used to derive algorithms
for many other problems that involve time series data with
loss functions and hypothesis set distinct from regression losses
and linear hypotheses considered in this section.}

\begin{figure}[!ht]
  \begin{center}
    \includegraphics[scale=0.2]{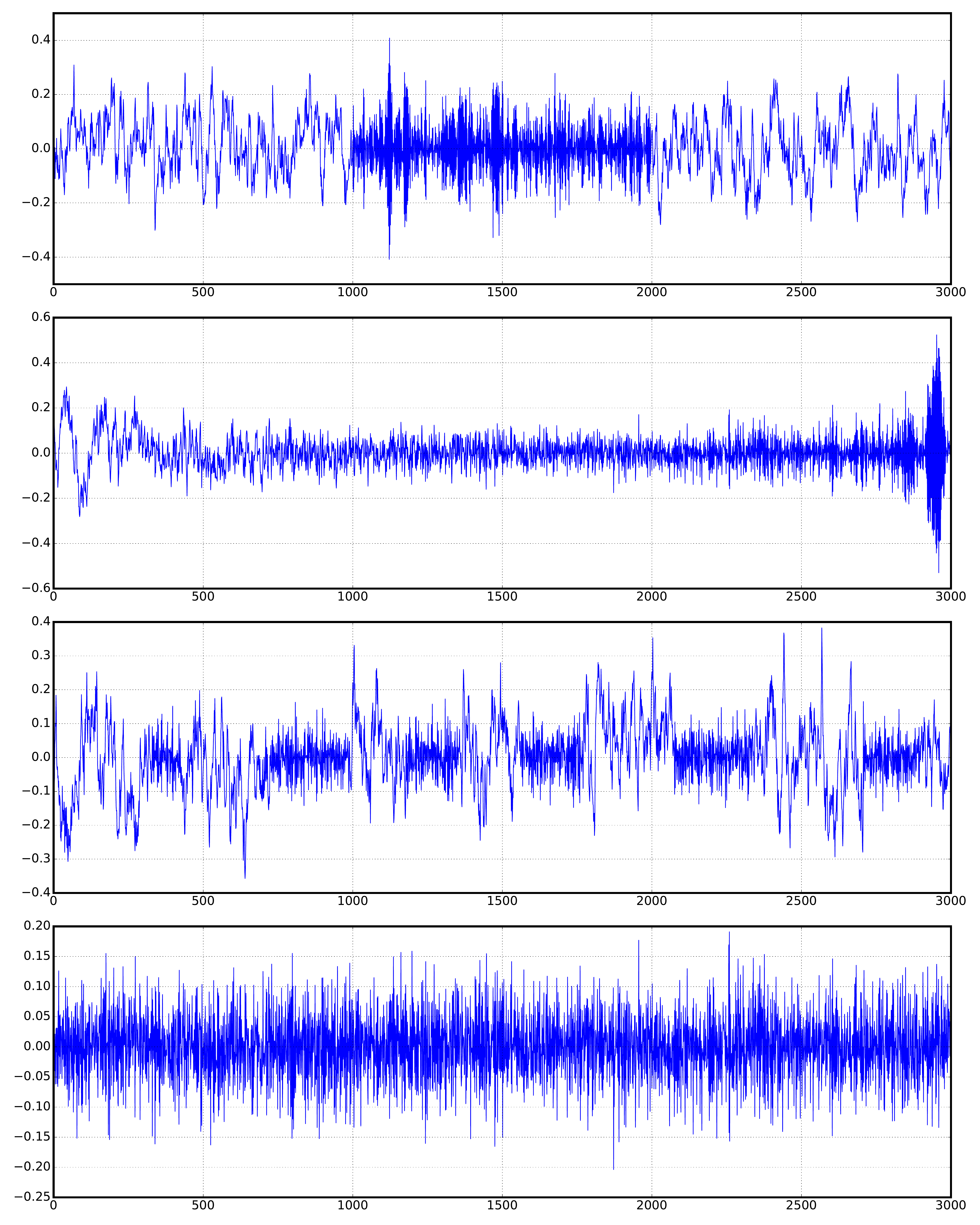}
  \end{center}%
\caption[Synthetic datasets.]{Synthetic datasets (top to bottom):
{\tt ads1}, {\tt ads2}, {\tt ads3}, {\tt ads4}. }
\label{fig:syn_all_ts}
\end{figure}

\begin{figure}[t]
  \begin{center}
    \includegraphics[scale=0.2]{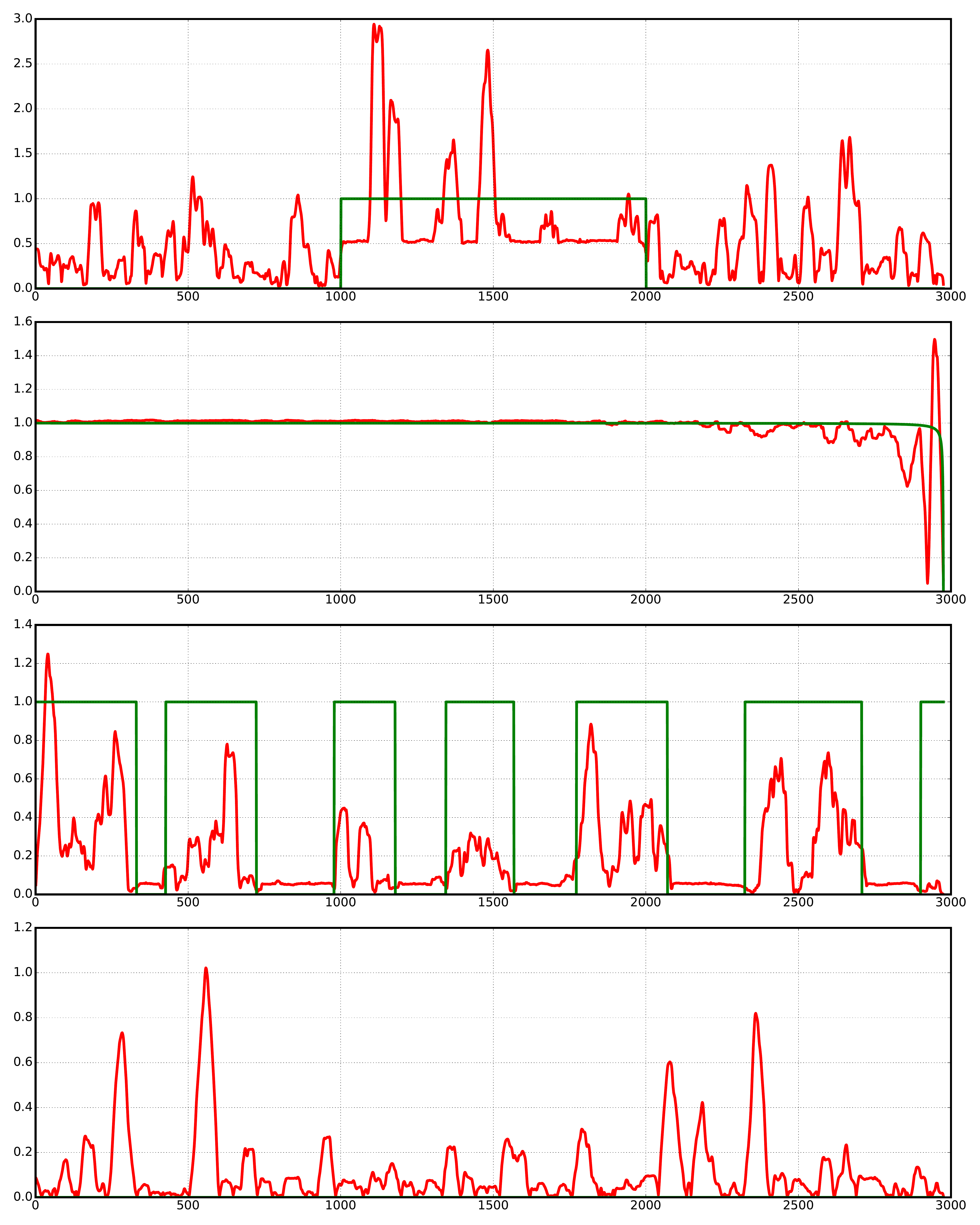}
  \end{center}%
\caption[True and estimated instantaneous discrepancies for synthetic
data.]{True (green) and estimated (red) instantaneous discrepancies
for synthetic data (top to bottom): {\tt ads1}, {\tt ads2},
{\tt ads3}, {\tt ads4}. }
\label{fig:all_ds}
\end{figure}

\begin{figure}[!ht]
  \begin{center}
    \includegraphics[scale=0.2]{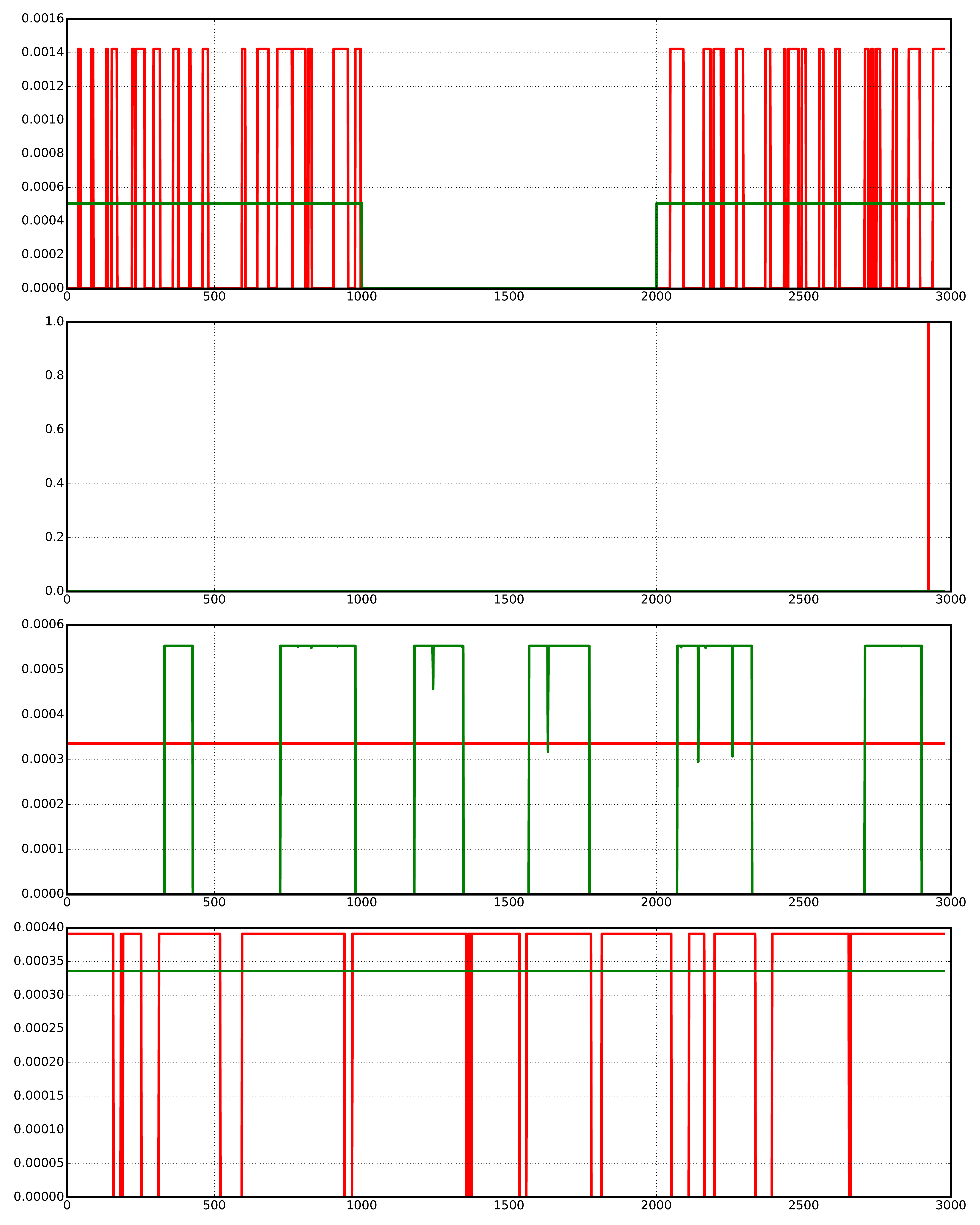}
  \end{center}%
\caption[Weights $\bq$ on synthetic datasets.]{
Weights $\bq$ chosen by DBF when used with 
true (green) and estimated (red) instantaneous discrepancies
for synthetic data (top to bottom):
{\tt ads1}, {\tt ads2},
{\tt ads3}, {\tt ads4}.}
\label{fig:all_qs}
\end{figure}

\begin{figure}[t]
  \begin{center}
    \includegraphics[scale=0.2]{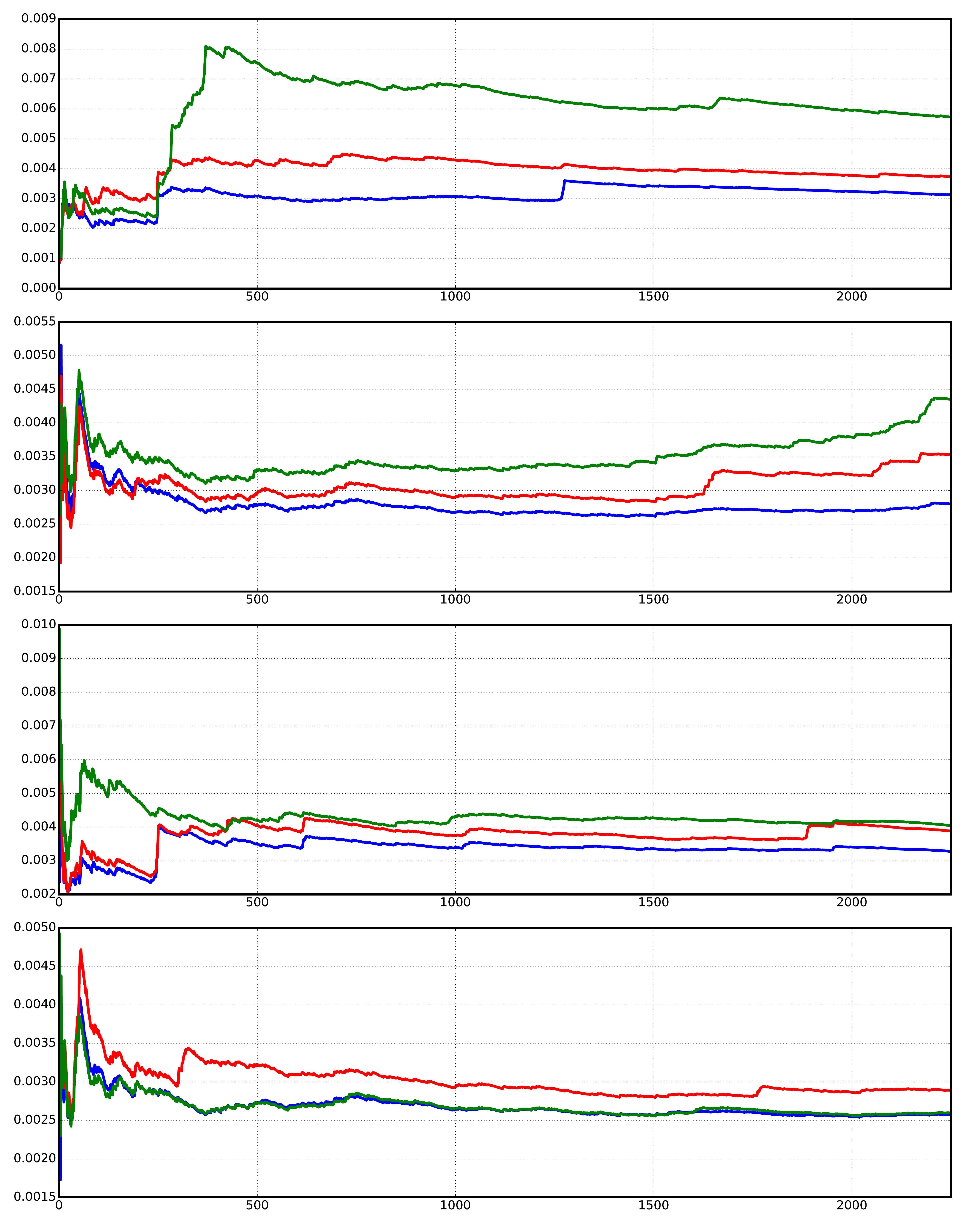}
  \end{center}%
\caption[Running MSE for synthetic data experiments.]{
Running MSE for synthetic data experiments
(top to bottom):
{\tt ads1}, {\tt ads2},
{\tt ads3}, {\tt ads4}.
For each time $t$ on the horizontal axis we plot
MSE up to time $t$ of {\tt tDBF} (blue),
{\tt eDBF} (red), {\tt ARIMA} (green).
}
\label{fig:all_ce}
\end{figure}

\begin{figure}[t]
  \begin{center}
    \includegraphics[scale=0.2]{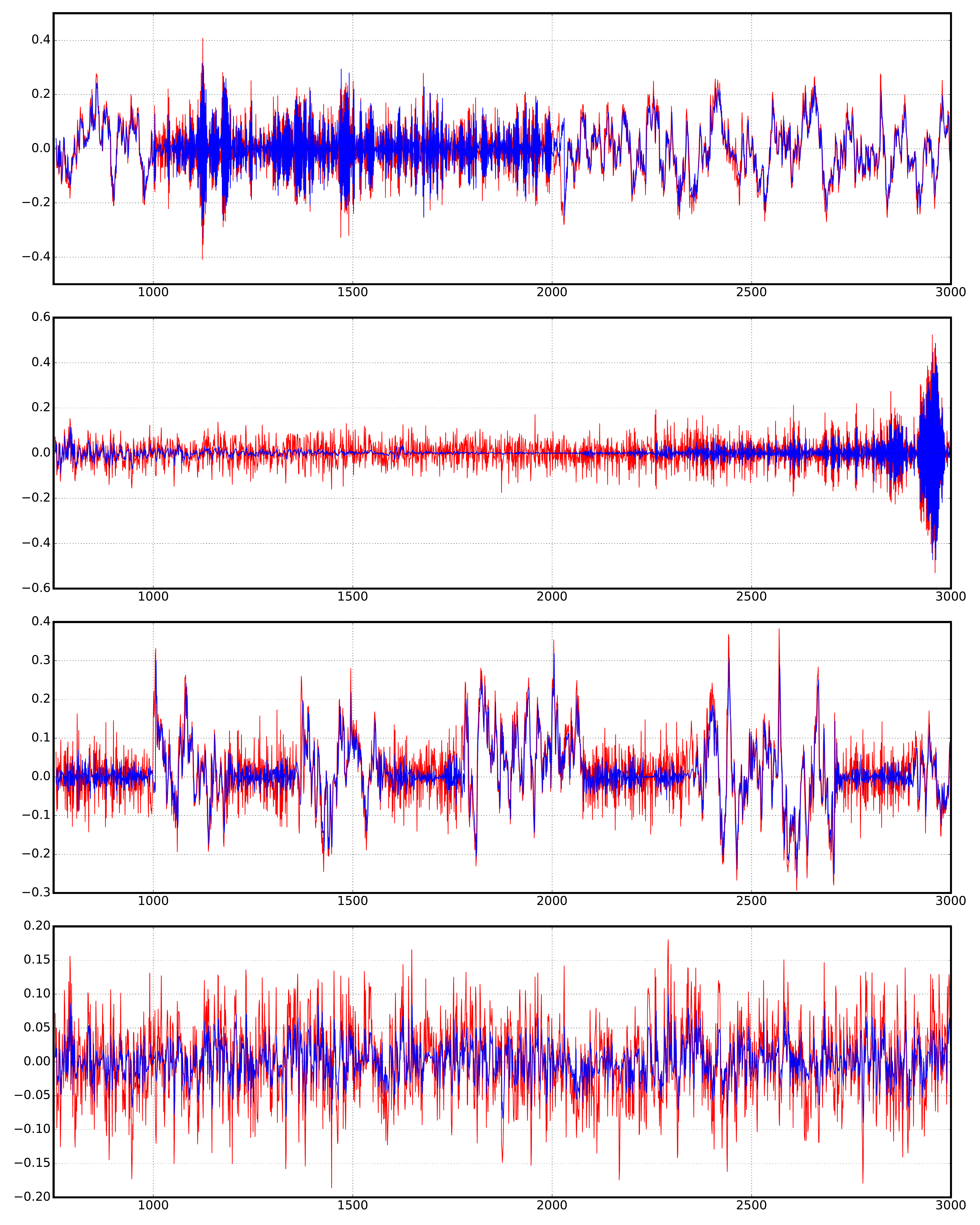}
  \end{center}%
\caption[True and {\tt tDBF}-forecasted time series on synthetic data.]{
True (red) and {\tt tDBF}-forecasted (blue) time series on synthetic data
(top to bottom):
{\tt ads1}, {\tt ads2},
{\tt ads3}, {\tt ads4}.
}
\label{fig:tdbf_ts}
\end{figure}

\begin{figure}[t]
  \begin{center}
    \includegraphics[scale=0.2]{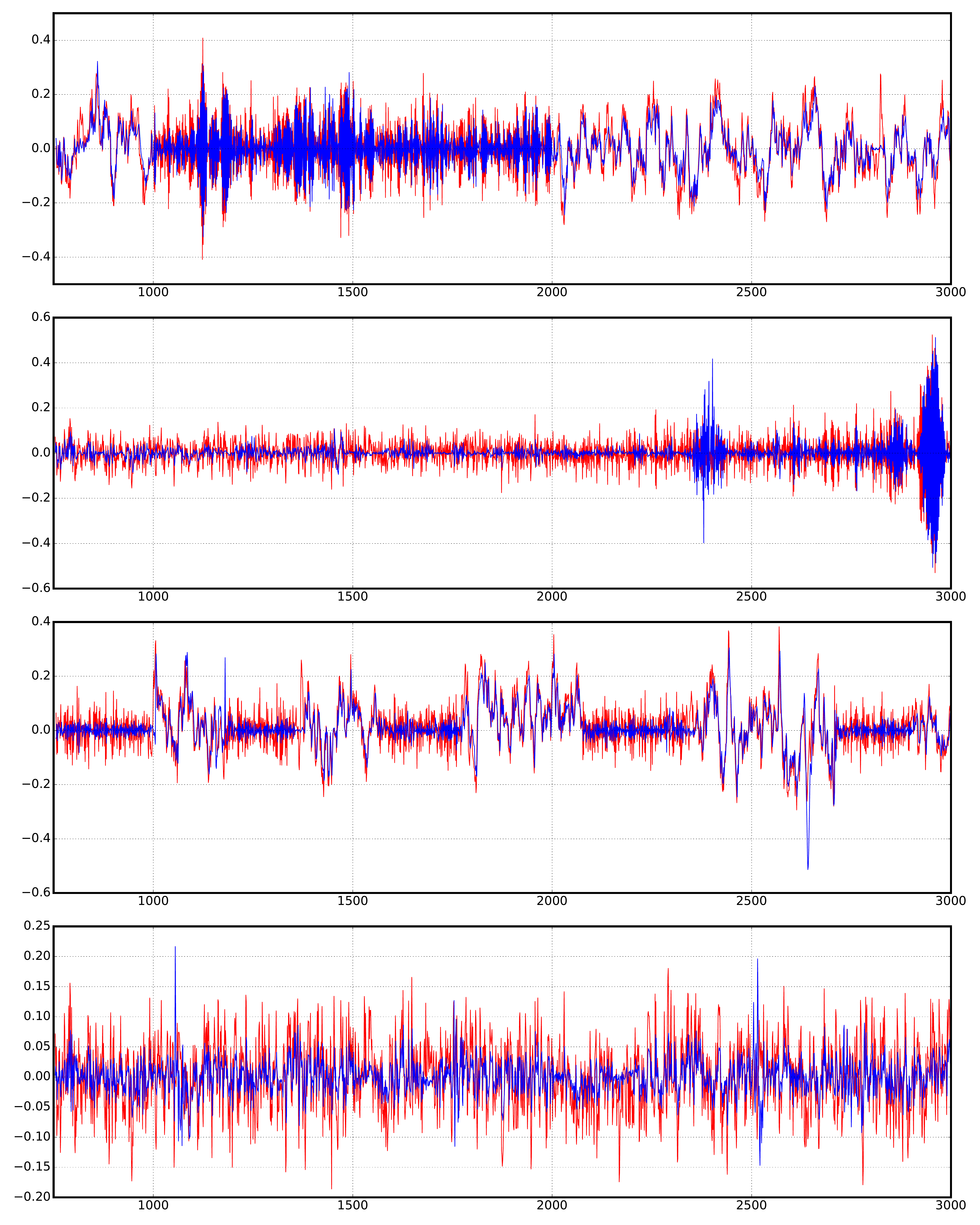}
  \end{center}%
\caption[True and {\tt eDBF}-forecasted time series on synthetic data.]{
True (red) and {\tt eDBF}-forecasted (blue) time series on synthetic data
(top to bottom):
{\tt ads1}, {\tt ads2},
{\tt ads3}, {\tt ads4}.
}
\label{fig:edbf_ts}
\end{figure}

\begin{figure}[t]
  \begin{center}
    \includegraphics[scale=0.2]{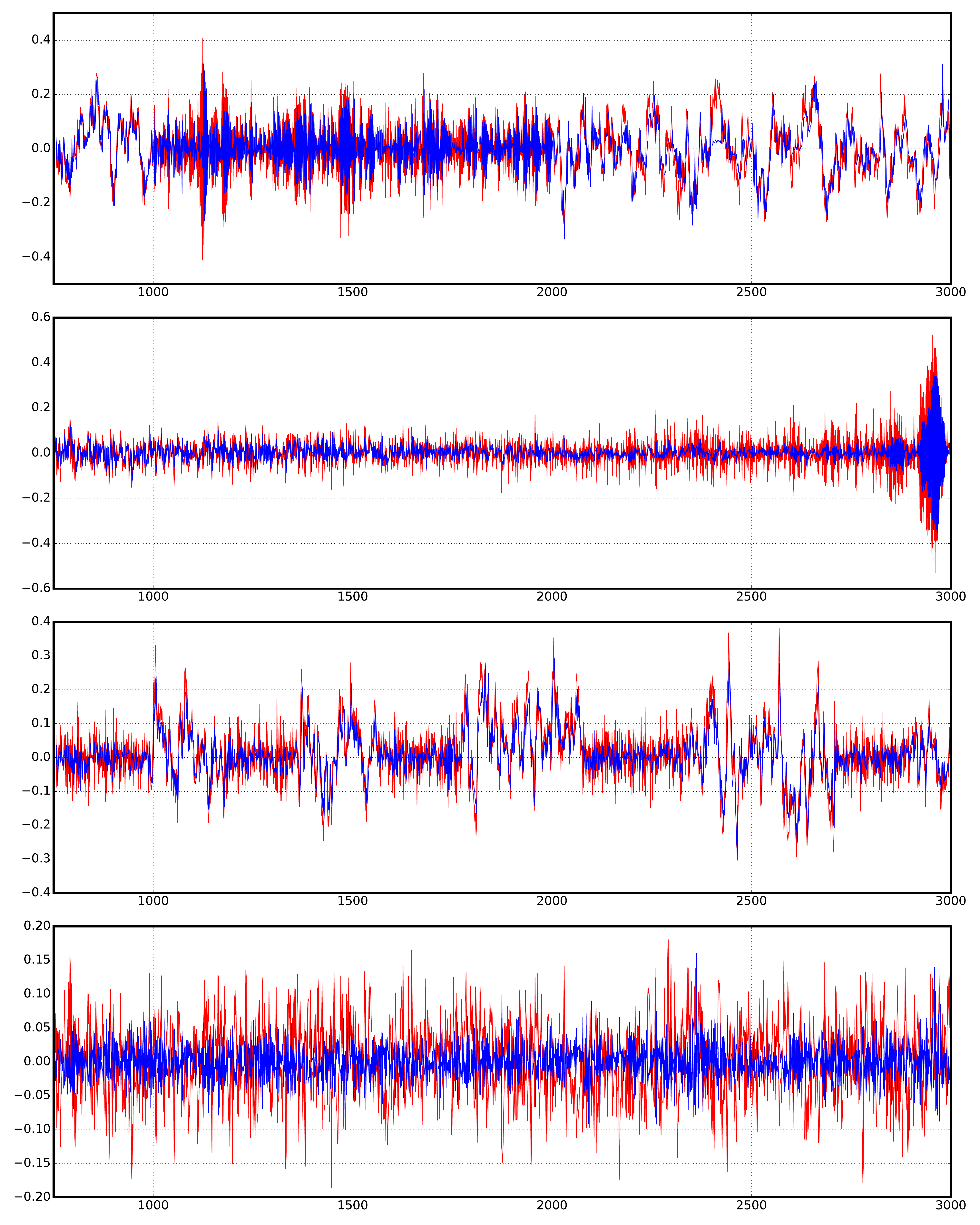}
  \end{center}%
\caption[True and {\tt ARIMA}-forecasted time series on synthetic data.]{
True (red) and {\tt ARIMA}-forecasted (blue) time series on synthetic data
(top to bottom):
{\tt ads1}, {\tt ads2},
{\tt ads3}, {\tt ads4}.
}
\label{fig:arima_ts}
\end{figure}

\section{Experiments}
\label{sec:experiments}

In this section, we present results of experiments
evaluating our algorithmic solutions on a number of synthetic
and real-world datasets.
In particular, we consider the one-stage algorithm presented
in Section~\ref{sec:algo} which is based on solving
the optimization problem in \eqref{eq:joint-opt}.
While solving this problem as opposed to
\eqref{eq:convex-opt} may result in a sub-optimal results,
this simplification allows us to use an alternating optimization method
described in Section~\ref{sec:algo}:
for a fixed $\bq$, problem
\eqref{eq:joint-opt} is a simple QP over $\bw$ and, for a fixed $\bw$,
the problem reduces to an LP in $\bq$. This iterative
scheme admits straightforward implementation using existing
QP and LP solvers. In the rest of this section,
we will refer to this algorithm as \emph{discrepancy-based forecaster}
(DBF).

We have chosen ARIMA models as a baseline comparator in our experiments.
These models
are standard and are commonly used in practice for forecasting
non-stationary time series.

We present two sets of experiments:
synthetic data experiments (Section~\ref{sec:syn-data-experiments})
and real-world data experiments (Section~\ref{sec:real-data-experiments}).

\subsection{Experiments with Synthetic Data}
\label{sec:syn-data-experiments}

\begin{table}[t]
\caption[Mean squared error (standard deviation) for synthetic data.]{
Mean squared error (standard deviation) for synthetic data. {\tt tDBF}
is DBF with true instantaneous discrepancies $d_t$ as its input.
{\tt eDBF} is DBF with estimated instantaneous discrepancies $d_t$
as its input. The results in bold are statistically significant
using one-sided paired $t$-test at $5\%$ level.}\label{tab:synthetic-exp}
\vskip 0.15in
\centering
%\scriptsize
\small
%\footnotesize
%\caption{Mean squared error (standard deviation) for synthetic data.}
\begin{tabular}{|c||c|c|c|}
\hline
Dataset     & {\tt tDBF} & {\tt eDBF} & {\tt ARIMA} \\ \hline
{\tt ads1} & $\mathbf{3.135 \times 10^{-3}}$
           & $3.743 \times 10^{-3}$
           & $5.723 \times 10^{-3} $ \\
           & $\mathbf{(7.504 \!\times\! 10^{-3})}$
           &  $(6.171 \times 10^{-3})$
           & $(10.143 \times 10^{-3})$ \\ \hline
{\tt ads2} & $\mathbf{2.800 \times 10^{-3}}$
           & $3.530 \times 10^{-3}$
           & $4.348 \times 10^{-3}$ \\ 
           & $\mathbf{(3.930 \!\times\! 10^{-3})}$
           & $(6.620 \times 10^{-3})$
           & $(6.770 \times 10^{-3})$ \\ \hline
{\tt ads3} & $\mathbf{3.282 \times 10^{-3}}$
           & $3.887 \times 10^{-3}$
           & $4.066 \times 10^{-3}$ \\
           & $\mathbf{(6.417 \!\times\! 10^{-3})}$
           & $(9.277 \times 10^{-3})$
           & $(6.122 \times 10^{-3})$ \\ \hline
{\tt ads4} & $2.573 \times 10^{-3}$
           & $2.889 \times 10^{-3}$
           & $2.593 \times 10^{-3}$ \\
           & $(3.516 \!\times\! 10^{-3})$
           & $(4.262 \times 10^{-3})$
           & $(3.578 \times 10^{-3})$ \\ \hline
\end{tabular}
%\vskip -0.05in
\end{table}

In this section, we present the results of our experiments on
some synthetic datasets. These experimental results allow us
to gain some further understanding of the discrepancy-based
approach to forecasting. In particular, it enables us
to study the effects of using estimated instantaneous discrepancies
instead of the true ones. 

We have used four
artificial datasets: {\tt ads1}, {\tt ads2}, {\tt ads3} and {\tt ads4}.
These datasets are generated using the following autoregressive processes:
\begin{align*}
\text{{\tt ads1: }} & Y_t = \alpha_t Y_{t-1} + \e_t, \quad\text{  $\alpha_t = -0.9$ if $t \in [1000, 2000]$ and $0.9$ otherwise}, \\
\text{{\tt ads2: }} & Y_t = \alpha_t Y_{t-1} + \e_t, \quad\,\,\, \alpha_t = 1 -  (t / 1500), \\
\text{{\tt ads3: }} & Y_t = \alpha_{i(t)} Y_{t-1} + \e_t, \quad \alpha_1 = -0.5 \text{ and } \alpha_2 = 0.9 \\
\text{{\tt ads4: }} & Y_t = -0.5 Y_{t-1} + \e_t,
\end{align*}
where $\e_t$ are independent Gaussian random variables with zero mean
and $\sigma = 0.05$. Note that $i(t)$ in the definition of {\tt ads3}
is a stochastic process on $\set{1,2}$, such that
$\P(i(t + s) = i | i(t+s-1) = \ldots = i(s) = i, i(s-1) \neq i)
= (0.99995)^t$. In other words, if the process $i(t)$ spend exactly
$\tau$ last time steps in state $i$, then at the next time step
it will stay in $i$ with probability $(0.99995)^\tau$ and will
move to a different state with probability $1 - (0.99995)^\tau$.

The first stochastic process ({\tt ads1}) is supposed to model
sudden abrupt changes in the data generating mechanism.
The scenario in which parameters of the data generating process
smoothly drift is modeled by {\tt ads2}. The setting in which
the changes can occur at random times is captured by {\tt ads3}.
Finally, {\tt ads4} is generated by a stationary random process.
See Figure~\ref{fig:syn_all_ts}.

For each dataset,
we have generated time series with $3\mathord,000$ sample points.
In all our experiments, we used the following protocol.
For each $t \in [750, 775, \ldots, 2995]$, $(y_1, \ldots, y_t)$
is used as a development set and $(y_{t+1}, \ldots, y_{t+25})$
is used as a test set. On the development set,
we first train each algorithm with different hyperparameter
settings on $(y_1, \ldots, y_{t-25})$ and then select
the best performing hyperparameters on $(y_{t-24}, \ldots, y_{t})$.
This set of hyperparameters is then used for training on the full development
set $(y_1, \ldots, y_t)$ and mean squared error (MSE) on
$(y_{t+1}, \ldots, y_{t+25})$ averaged over all
$t \in [750, 775, \ldots, 2995]$ is reported.

Recall that DBF algorithm requires
two regularization hyperparameters $\lambda_1$
and $\lambda_2$. We optimized these parameters over following sets of values
$\set{10^{-3}, 10^{-4}, 10^{-5}, 10^{-6}}$ and
$\set{100,\, 10,\, 1,\, 0.1,\, 0.05,\, 0.01,\, 0.001,\, 0}$
for $\lambda_1$ and $\lambda_2$ respectively.

ARIMA models have three hyperparameters $p, d$ and $q$
that we also set via the validation procedure described above.
Recall that ARIMA($p,d,q$) is
a generative model defined by the following autoregression:
\begin{align*}
\Bigg(1- \sum_{i=0}^{p-1} \phi_i \mathfrak{L}^i\Bigg)
(1- \mathfrak{L})^d Y_{t+1} =
\Bigg(1 + \sum_{j=1}^q \theta_j \mathfrak{L}^j\Bigg) \e_t.
\end{align*}
where $\mathfrak{L}$ is a lag operator, that is,
$\mathfrak{L} Y_t = Y_{t-1}$. Therefore, validation
over $(p,d,q)$ is equivalent to validation over different sets
of features used to train the model. For instance,
$(p,d,q) = (3,0,0)$ means that we are using $(y_{t-3}, y_{t-2}, y_{t-1})$
as our features while $(p,d,q) = (2,1,0)$ corresponds
to $(y_{t-3}-y_{t-2}, y_{t-2} - y_{t-1})$.
For DBF, we fix feature vectors to be $(y_{t-3}, y_{t-2}, y_{t-1})$.
For ARIMA, we optimize over $p,d,q \in \set{0,1,2}^3$.
We use maximum likelihood approach to estimate unknown
parameters of ARIMA models.

Finally, observe that discrepancy estimation procedure
discussed in Section~\ref{sec:algo} also requires
a hyperparameter $s$ representing the length of the most recent history
of the stochastic process. We did not make an attempt to optimize
this parameter and in all of our experiments we set $s=20$. 

\begin{table}[t]
\caption[Real-world datasets statistics.]{Real-world datasets statistics}\label{tab:datasets}
\vskip 0.15in
\centering
%\scriptsize
\small
%\footnotesize
\begin{tabular}{|c||c|c|}
\hline
Dataset      & URL & Size \\ \hline
{\tt bitcoin} & {\tt {\scriptsize https://www.quandl.com/data/BCHARTS/BTCNCNY}} 
           & 1705 \\
{\tt coffee} &  {\tt {\scriptsize https://www.quandl.com/data/COM/COFFEE\_BRZL}}
           & 2205 \\
{\tt eur/jpy} &  {\tt {\scriptsize https://www.quandl.com/data/ECB/EURJPY}} 
           & 4425 \\
{\tt jpy/usd} & {\tt {\scriptsize http://data.is/269FpLF}} 
           & 4475 \\
{\tt mso} & {\tt {\scriptsize http://data.is/269F3EV}}
           & 1235 \\ 
{\tt silver} & {\tt {\scriptsize https://www.quandl.com/data/COM/AG\_EIB}} 
           & 2251 \\
{\tt soy} & {\tt {\scriptsize https://www.quandl.com/data/COM/SOYB\_1}} 
           & 2218 \\
{\tt temp} & {\tt {\scriptsize http://data.is/1qlX2AN}}
           & 3649 \\ \hline
\end{tabular}
%\vskip -0.05in
\end{table}

The results of our experiments are presented in Table~\ref{tab:synthetic-exp}.
We have compared DBF with true discrepancies as its input {\tt tDBF},
DBF with estimated discrepancies as its input {\tt eDBF} and ARIMA.
In all experiments with non-stationary processes
({\tt ads1}, {\tt ads2}, {\tt ads3}), {\tt tDBF} performs
better than both {\tt eDBF} and ARIMA. Similarly, {\tt eDBF}
outperforms ARIMA on the same datasets. These results
are statistically significant at $5\%$-level using one-sided paired
$t$-test. Figure~\ref{fig:all_ce} illustrates the dynamics of MSE
as a function of time $t$ for all three algorithms on all of the synthetic
datasets. Figures~\ref{fig:tdbf_ts},~\ref{fig:edbf_ts} and~\ref{fig:arima_ts}
show forecasted and true time series for each dataset and algorithm.  

Our results suggest that accurate discrepancy estimation can lead
to a significant improvement in performance. We present the results
of discrepancy estimation for our experiments in
Figure~\ref{fig:all_ds} and Figure~\ref{fig:all_qs} shows the
corresponding weights $\bq$ chosen by DBF.

\subsection{Experiments with Real-World Data}
\label{sec:real-data-experiments}

\begin{table}[t]
\caption[Mean squared error (standard deviation) for real-world data.]{
Mean squared error (standard deviation) for real-world data.
The results in bold are statistically significant
using one-sided paired $t$-test at $5\%$ level.}\label{tab:real-world-exp}
\vskip 0.15in
\centering
%\scriptsize
\small
%\footnotesize
%\caption{}
\begin{tabular}{|c||c|c|}
\hline
Dataset     & {\tt DBF} & {\tt ARIMA} \\ \hline
{\tt bitcoin} & $\mathbf{4.400 \times 10^{-3}}$
           & $4.900 \times 10^{-3}$ \\
           & $\mathbf{(26.500 \!\times\! 10^{-3})}$
           &  $(29.990 \times 10^{-3})$ \\ \hline
{\tt coffee} & $\mathbf{3.080 \times 10^{-3}}$
           & $3.260 \times 10^{-3}$ \\
           & $\mathbf{(6.570 \!\times\! 10^{-3})}$
           & $(6.390 \times 10^{-3})$ \\ \hline
{\tt eur/jpy} & $\mathbf{7.100 \times 10^{-5}}$
           & $7.800 \times 10^{-5}$ \\
           & $\mathbf{(16.900 \!\times\! 10^{-5})}$
           & $(24.200 \times 10^{-5})$ \\ \hline
{\tt jpy/usd} & $\mathbf{9.770 \times 10^{-1}}$
           & $ 10.004 \times 10^{-1} $ \\
           & $\mathbf{(25.893 \!\times\! 10^{-1})}$
           & $(27.531 \times 10^{-1})$ \\ \hline
{\tt mso} & $32.876 \times 10^{0}$
           & $32.193 \times 10^{0}$ \\ 
           & $(55.586 \!\times\! 10^{0})$
           & $(51.109 \times 10^{0})$ \\ \hline
{\tt silver} & $\mathbf{7.640 \times 10^{-4}}$
           & $34.180 \times 10^{-4}$ \\
           & $\mathbf{(46.65 \!\times\! 10^{-4})}$
           & $(158.090 \times 10^{-4})$ \\ \hline
{\tt soy} & $5.071 \times 10^{-2}$
           & $5.003 \times 10^{-2}$ \\
           & $(9.938 \!\times\! 10^{-2})$
           & $(10.097 \times 10^{-2})$ \\ \hline
{\tt temp} & $6.418 \times 10^{0}$
           & $6.461 \times 10^{0}$ \\
           & $(9.958 \!\times\! 10^{0})$
           & $(10.324 \times 10^{0})$ \\ \hline
\end{tabular}
\end{table}

In this section, we present the results of our experiments
with real-world datasets. For our experiments,
we have chosen eight time series from
different domains:
currency exchange rates ({\tt bitcoin}, {\tt eur/jpy}, {\tt jpy/usd}),
commodity prices ({\tt coffee}, {\tt soy}, {\tt silver}) and
meteorology ({\tt mso}, {\tt temp}).
Further details of these datasets are summarized in Table~\ref{tab:datasets}.

In all our experiments, we used the same protocol as in the
previous section.\ignore{
For each $t \in [750, 775, \ldots$, $(y_1, \ldots, y_t)$
is used as a development set and $(y_{t+1}, \ldots, y_{t+25})$
is used as a test set. On the development set,
we first train each algorithm with different hyperparameter
settings on $(y_1, \ldots, y_{t-25})$ and then select
the best performing hyperparameters on $(y_{t-24}, \ldots, y_{t})$.
This set of hyperparameters is then used for training on the full development
set $(y_1, \ldots, y_t)$ and mean squared error (MSE) on
$(y_{t+1}, \ldots, y_{t+25})$ averaged over all
$t \in [750, 775, \ldots, 2995]$ is reported.}
The range of the hyperparameters for both DBF and ARIMA
is also the same as in previous section. Note that
since true discrepancies are unknown, we only present the results
for DBF with discrepancies estimated from data.

The results of our experiments are summarized in
Table~\ref{tab:real-world-exp}. Observe that DBF outperforms ARIMA
on 5 out of 8 datasets.

Figure~\ref{fig:all_ce_real} illustrates the dynamics of MSE
as a function of time $t$ for all three algorithms on all of the synthetic
datasets. These results
are statistically significant at $5\%$-level using one-sided paired
$t$-test. There is not statistical difference in the performance of ARIMA
and DBF on the rest of the datasets.
Figures~\ref{fig:dbf_ts_real} and~\ref{fig:arima_ts_real}
show forecasted and true time series for each dataset for DBF and ARIMA
respectively. 

Our results suggest that discrepancy-based approach to prediction of
non-stationary time series may lead to improved performance
compared to other traditional approaches such as ARIMA.

\begin{figure}[t]
  \begin{center}
    \includegraphics[scale=0.2]{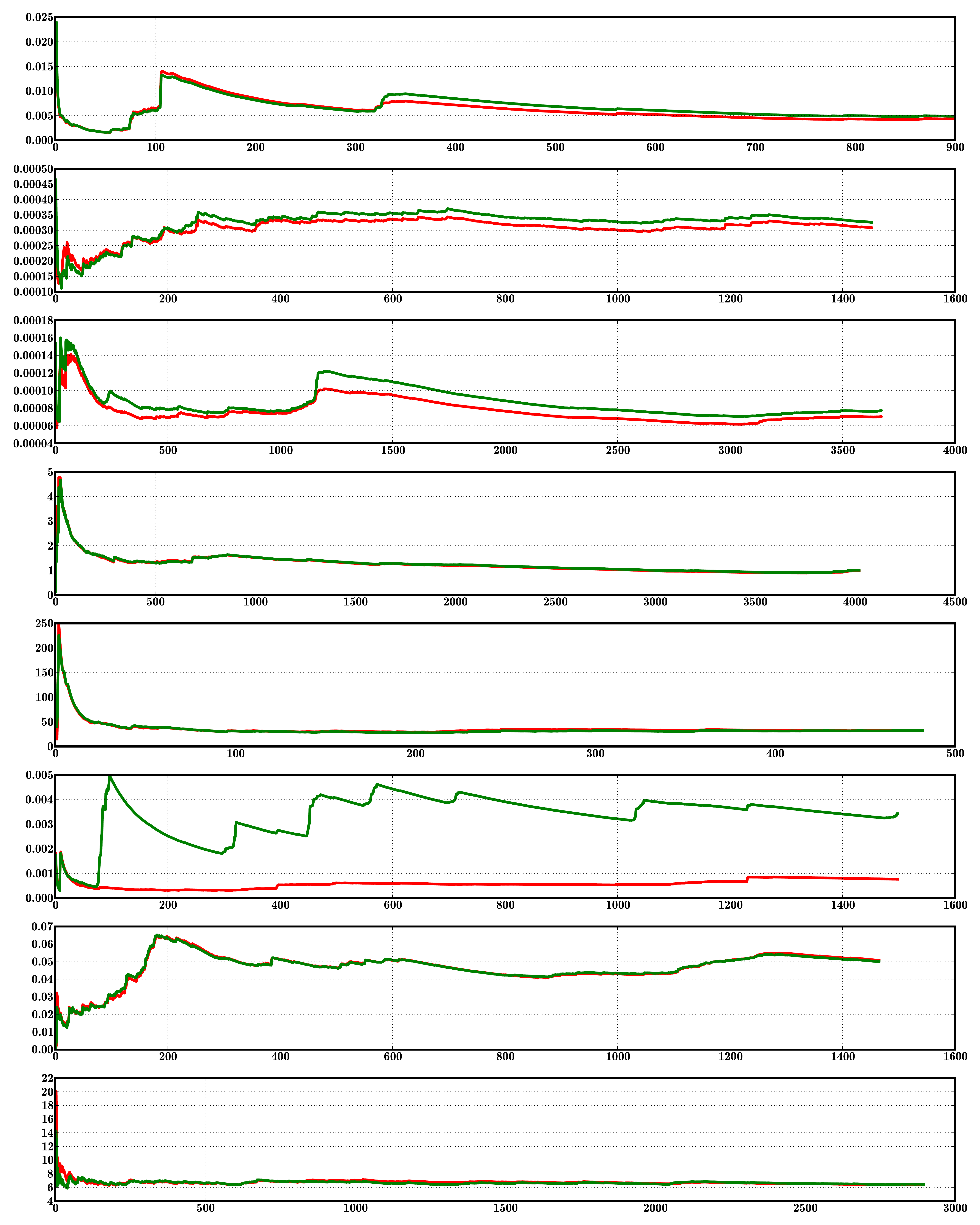}
  \end{center}%
\caption[Running MSE for real-world data experiments.]{
Running MSE for real-world data experiments (top to bottom): 
{\tt bitcoin}, {\tt coffee}, {\tt eur/jpy}, {\tt jpy/usd},
{\tt mso}, {\tt silver}, {\tt soy}, {\tt temp}.
For each time $t$ on the horizontal axis we plot
MSE up to time $t$ of {\tt DBF} (red) and {\tt ARIMA} (green).
}
\label{fig:all_ce_real}
\end{figure}

\begin{figure}[t]
  \begin{center}
    \includegraphics[scale=0.2]{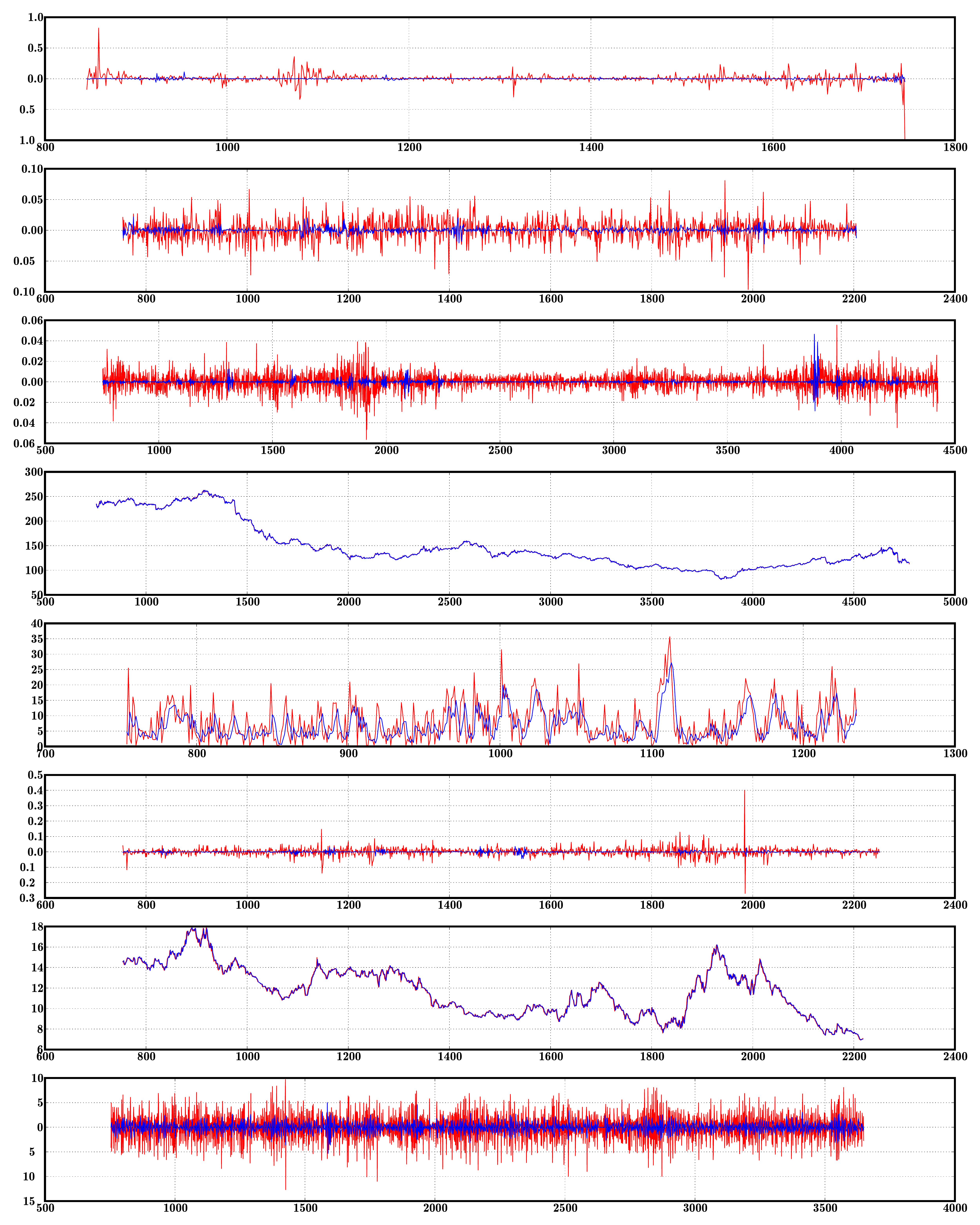}
  \end{center}%
\caption[True and {\tt DBF}-forecasted time series on real-world data.]{
True (red) and {\tt DBF}-forecasted (blue) time series on
real-world data experiments (top to bottom): 
{\tt bitcoin}, {\tt coffee}, {\tt eur/jpy}, {\tt jpy/usd},
{\tt mso}, {\tt silver}, {\tt soy}, {\tt temp}.} 
\label{fig:dbf_ts_real}
\end{figure}

\begin{figure}[t]
  \begin{center}
    \includegraphics[scale=0.2]{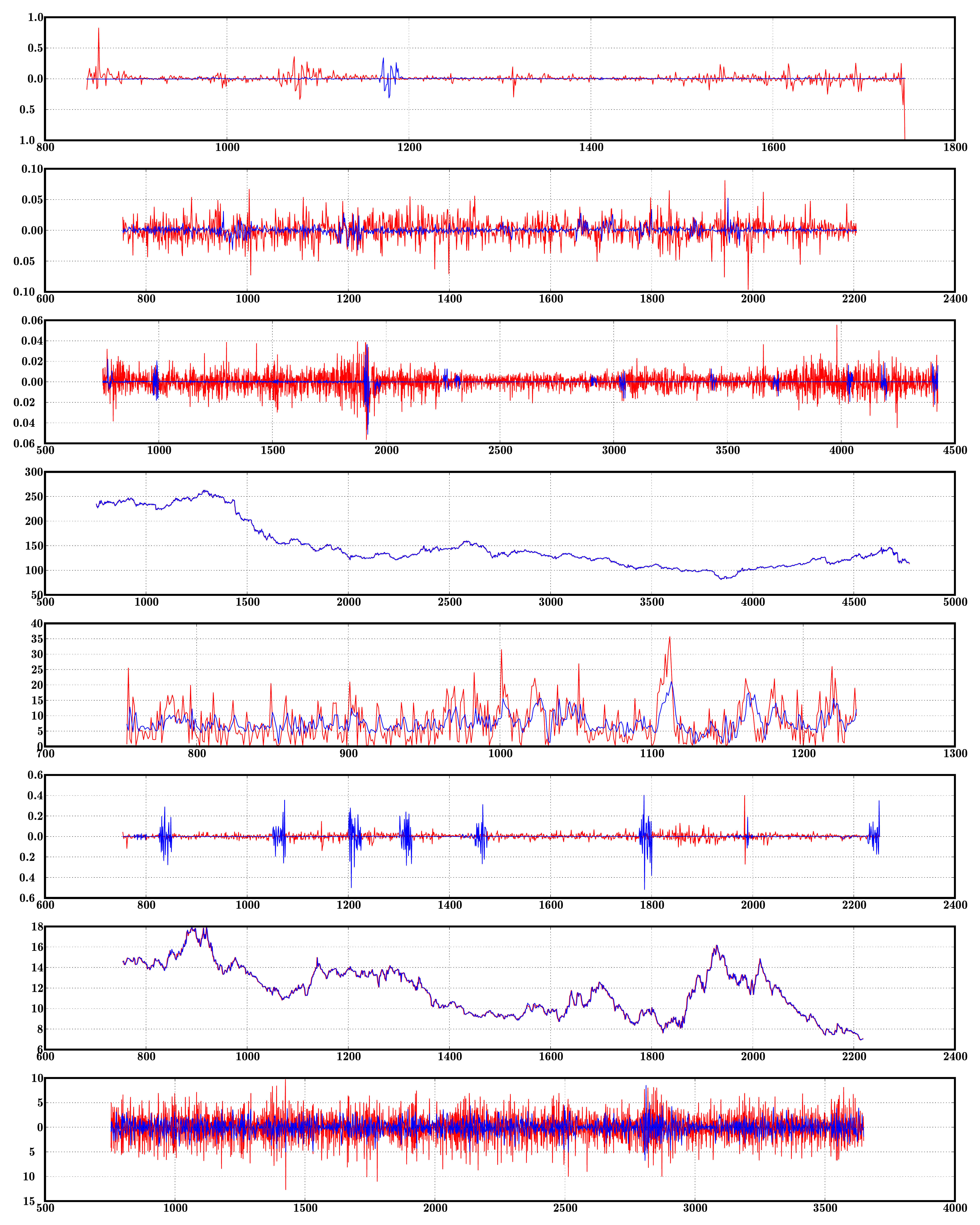}
  \end{center}%
\caption[True and {\tt ARIMA}-forecasted time series on real-world data.]{
True (red) and {\tt ARIMA}-forecasted (blue) time series on
real-world data experiments (top to bottom): 
{\tt bitcoin}, {\tt coffee}, {\tt eur/jpy}, {\tt jpy/usd},
{\tt mso}, {\tt silver}, {\tt soy}, {\tt temp}.} 
\label{fig:arima_ts_real}
\end{figure}

\begin{figure}[t]
  \begin{center}
    \includegraphics[scale=0.2]{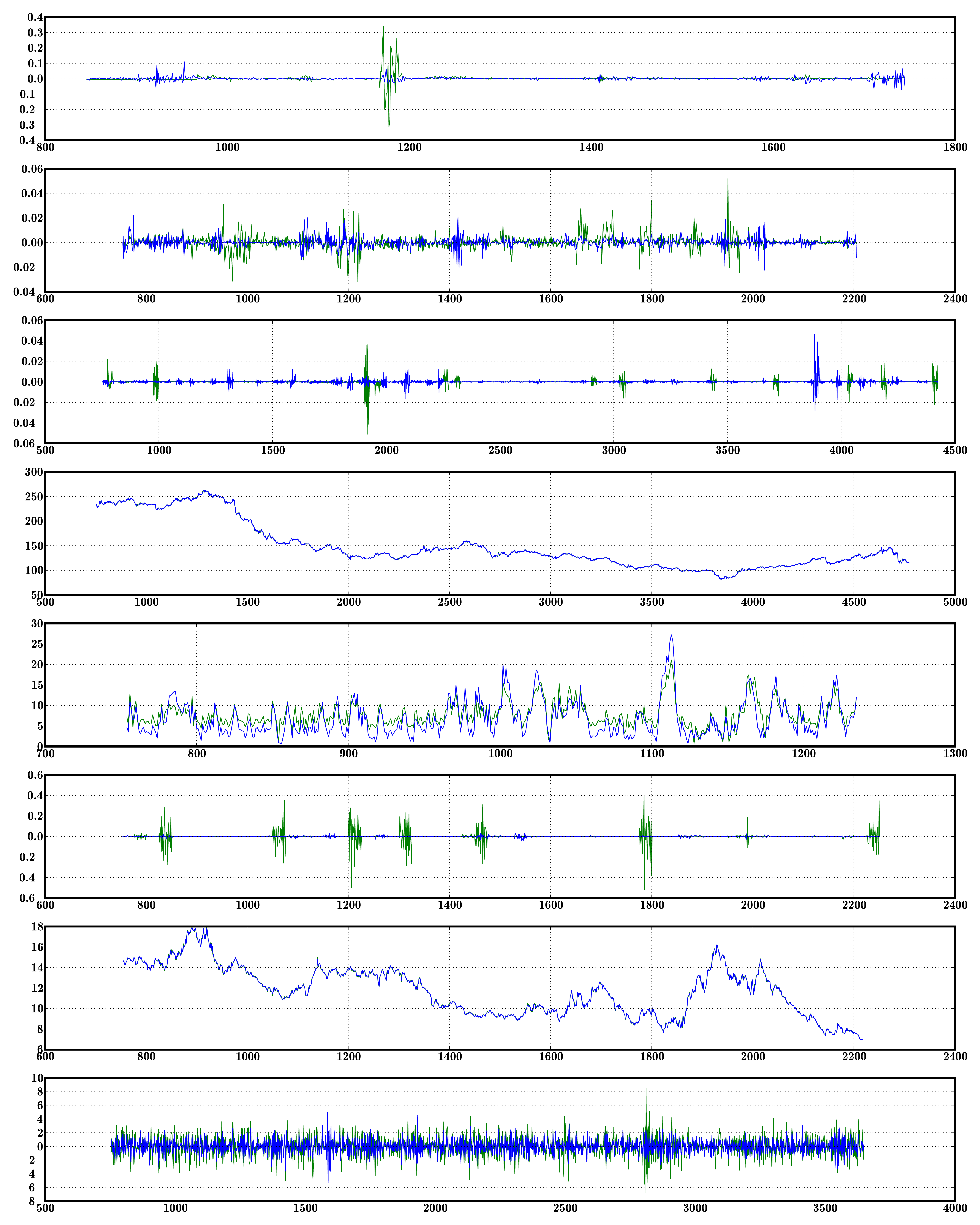}
  \end{center}%
\caption[{\tt DBF}- and {\tt ARIMA}-forecasted
time series on real-world data.]{
{\tt DBF} (blue) and {\tt ARIMA}-forecasted (green) time series on
real-world data experiments (top to bottom): 
{\tt bitcoin}, {\tt coffee}, {\tt eur/jpy}, {\tt jpy/usd},
{\tt mso}, {\tt silver}, {\tt soy}, {\tt temp}.} 
\label{fig:dbf_arima_ts_real}
\end{figure}

\newpage

\section{Conclusion}
\label{sec:conclusion}

We presented a general theoretical analysis of learning in the broad
scenario of non-stationary non-mixing processes, the realistic setting
for a variety of applications. We discussed in detail several
algorithms benefiting from the learning guarantees presented. Our
theory can also provide a finer analysis of several existing
algorithms and help devise alternative principled learning algorithms.

\acks{
This work was partly funded by NSF CCF-1535987 and IIS-1618662.}

\appendix

\section{Proofs}
\label{sec:proofs}

\begin{lemma}
\label{lm:sym}
Given a sequence of random variables $\bZ_1^T$ with joint distribution
$\bp$, let ${\bZ'}_1^T$ be a decoupled tangent sequence. Then, for any
measurable function $G$, the following equality holds
\begin{equation}
\label{eq:seq-rademacher-bound}
\E\Bigg[G\Big(\sup_{f \in \cF} \sum_{t = 1}^T q_t (f(Z'_t) - f(Z_t))\Big)\Bigg]
=
\E_{\bs} \E_{\bz \sim T(\bp)}\Big[G\Big(\sup_f \sum_{t = 1}^T
\sigma_t q_t (f(\bz'_t(\bs)) - f(\bz_t(\bs))) \Big)\Big].
\end{equation}
\end{lemma}
The result also holds with the absolute value around the sums
in \eqref{eq:seq-rademacher-bound}.

\begin{proof}
  The proof follows an argument in the proof of Theorem~3 of
  \citep{RakhlinSridharanTewari2011}.  We only need to check that
  every step holds for an arbitrary weight vector $\bq$, in lieu of
  the uniform distribution vector $\bu$, and for an arbitrary
  measurable function $G$, instead of the identity function.  Observe
  that we can write the left-hand side of
  \eqref{eq:seq-rademacher-bound} as
\begin{align*}
\E\Big[ G\Big(\sup_{f \in \cF}  \Sigma( \bs) \Big) \Big] =
\E_{Z_1, Z'_1 \sim \bp_1} \E_{Z_2, Z'_2 \sim \bp_2(\cdot|Z_1)}
\cdots \E_{Z_T, Z'_T \sim \bp_T(\cdot|\bZ_1^{T-1})}
\Big[G\Big(\sup_{f \in \cF}  \Sigma(\bs) \Big) \Big],
\end{align*}
where $\bs = (1, \ldots, 1) \in \set{\pm 1}^T$ and
$\Sigma(\bs) = \sum_{t = 1}^T \sigma_t q_t (f(Z'_t) - f(Z_t))$.
Now, by definition of decoupled tangent sequences,
the value of the last expression is unchanged if
we swap the sign of any $\sigma_{i - 1}$ to $-1$ since that
is equivalent to permuting $Z_i$ and $Z'_i$. Thus, the last expression 
is in fact equal to
\begin{align*}
\E_{Z_1, Z'_1 \sim \bp_1} \E_{Z_2, Z'_2 \sim \bp_2(\cdot|S_1(\s_1))}
\cdots \E_{Z_T, Z'_T \sim \bp_T(\cdot|S_1(\s_1), \ldots, S_{T-1}(\s_{T-1}))}
\Big[G\Big(\sup_{f \in \cF}  \Sigma(\bs) \Big) \Big]
\end{align*}
for any sequence $\bs \in \set{\pm 1}^T$, where $S_t(1) = Z_t$ and
$Z'_t$ otherwise. Since this equality holds for any $\bs$, it also
holds for the mean with respect to uniformly distributed
$\bs$. Therefore, the last expression is equal to
\begin{align*}
\E_{\bs} \E_{Z_1, Z'_1 \sim \bp_1} \E_{Z_2, Z'_2 \sim \bp_2(\cdot|S_1(\s_1))}
\cdots \E_{Z_T, Z'_T \sim \bp_T(\cdot|S_1(\s_1), \ldots, S_{T-1}(\s_{T-1}))}
\Big[G\Big(\sup_{f \in \cF}  \Sigma(\bs) \Big) \Big].
\end{align*}
This last expectation coincides with the expectation
with respect to drawing a random tree $\bz$ from $T(\bp)$
(and its tangent tree $\bz'$) and
a random path $\bs$ to follow in that tree. That is,
the last expectation is equal to
\begin{align*}
\E_{\bs} \E_{\bz \sim T(\bp)}\Big[G\Big(\sup_f \sum_{t = 1}^T
\sigma_t q_t (f(\bz'_t(\bs)) - f(\bz_t(\bs))) \Big)\Big],
\end{align*}
which concludes the proof.
\end{proof}

\ignore{
\begin{reptheorem}{th:discrepancy}
Let $\bZ_1^T$ be a sequence of random variables.  Then, for any
$\d > 0$, with probability at
least $1 - \delta$, the following holds for all $\alpha > 0$:
\begin{align*}
\sup_{f \in \cF} \Bigg( \sum_{t = 1}^T(p_t - q_t)
 \E[f(Z_t)|\bZ_1^{t-1}]\Bigg) \leq
 & \sup_{f \in \cF} \Bigg( \sum_{t = 1}^T (p_t - q_t) f(Z_t) \Bigg) \\
 & + \alpha +
M \|\bq - \bp\|_2 \sqrt{\log{\frac{\E_{\bz \sim T(\bp)}[ \cN_1(\alpha, \cG, \bz) ]}{\d'}}},
\end{align*}
where $\bp$ is the distribution the uniform on the last $s$ points.
\end{reptheorem}

\begin{proof}
First, observe that
\begin{align*}
\sup_{f \in \cF} \Bigg( \sum_{t = 1}^T(p_t - q_t)
 \E[f(Z_t)|\bZ_1^{t-1}]\Bigg) -
 & \sup_{f \in \cF} \Bigg( \sum_{t = 1}^T (p_t - q_t) f(Z_t) \Bigg)\\ &\leq
\sup_{f \in \cF} \Bigg( \sum_{t = 1}^T(p_t - q_t)(
 \E[f(Z_t)|\bZ_1^{t-1}] -  f(Z_t)) \Bigg).
\end{align*}
The result then follows using similar arguments to those used in the proof of Theorem~\ref{th:bound}.
\end{proof}}

\begin{reptheorem}{th:stability-coef}
Let $K$ be a positive definite symmetric kernel such that
$r^2 = \sup_x K(x,x)$ and let $L$ be a convex and $\s$-admissible loss
function. Let $\bq = (q_1, \ldots, q_T)$ be any
non-negative weight vector.
Then, the kernel-based regularization algorithm defined by the
minimization of $F_{\bZ_1^T}(h, \bq)$ in \eqref{eq:q-kernel-based-reg-obj}
is $\beta$-stable with $\beta \leq \frac{\s^2 r^2 \|\bq\|_\infty}{\lambda}$.
\end{reptheorem}

\begin{proof}
Let $S = \bZ_1^T$ and $S'$ be a sample that difers from $S$ by exactly one
point, say $Z'_t$. Assume $h$ and $h'$ are minimizers of
$F_S = F_S(\cdot, \bq)$ and $F_{S'} = F_{S'}(\cdot, \bq)$ respectively.
We let $B_{F_S}$ denote the generalized Bregman divergence defined
by $F_S$, that is,
\begin{align*}
B_{F_S}(h_1 \parallel h_2) = F_S(h_1) - F_S(h_2) -
\langle h_1 - h_2, \d F_S(h_2) \rangle_\cH,
\end{align*}
where $\d F_S(h)$ is denotes any element of subgradient of $F_S$ at $h$
such that
$\d (\sum_{t=1}^T q_t L(h, Z_t)) = \d F_S(h) - \lambda \nabla \|h\|_\cH^2$
and $\d F_s(h) = 0$ whenever $h$ is a minimizer of $F_S$. Note
that this implies that $B_{F_S} = B_{\hh R_S} + \lambda B_N$,
where $N(h) = \|h\|_\cH^2$ and $\hh R_S = \sum_{t=1}^T q_t L(h, Z_t)$.
Then, since generalized Bregman divergence is non-negative,  we can write,
\begin{align*}
B_{F_S}(h' \parallel h) + B_{F_{S'}}(h \parallel h')
\geq \lambda
B_{N}(h' \parallel h) + B_{F_{N}}(h \parallel h').
\end{align*}
Observe that $B_N(h'\parallel h) + B_N(h\parallel h')
= - \langle h'-h, 2h \rangle - \langle h - h', 2h' \rangle = 2\|h' - h\|_\cH^2$. Let $\Delta h = h' - h$. Then it follows that
\begin{align*}
2 \lambda \|\Delta h \|_\cH^2 &\leq
B_{F_S}(h' \parallel h) + B_{F_{S'}}(h \parallel h') \\
&= F_S(h') - F_S(h) -
\langle h' - h, \d F_{S}(h) \rangle_\cH
+ F_{S'}(h) - F_{S'}(h') -
\langle h - h, \d F_{S'}(h') \rangle_\cH \\
&= F_S(h') - F_S(h) + F_{S'}(h) - F_{S'}(h') \\
&= \hh R_S(h') - \hh R_S(h) + \hh R_{S'}(h) - \hh R_{S'}(h'),
\end{align*}
where the first equality uses the definition of the generalized Bregman
divergence, second equality follows from the fact that $h$ and $h'$
are minimizers and last equality follows from the definition of
$F_S$ and $F_{S'}$. By $\s$-admissibility of $L$ and the fact that
$S$ and $S'$ differ by exactly one point it follows that
\begin{align*}
2 \lambda \|\Delta h \|_\cH^2 \leq
q_t ( L(h', Z_t) - L(h, Z_t) + L(h, Z'_t) - L(h', Z'_t) )
\leq \s q_t (|\Delta h(X_t)| + |\Delta h(X'_t)|).  
\end{align*}
Applying the reproducing kernel property and Cauchy-Schwartz inequality,
for all $x \in \cX$,
\begin{align*}
\Delta(x) = \langle \Delta h, K(x, \cdot) \rangle_\cH
\leq \|\Delta h\|_\cH \|K(x,\cdot)\|_\cH \leq r \|\Delta h\|_\cH.
\end{align*}
It follows that $\|\Delta h\|_\cH \leq \frac{\s r \| \bq \|_\infty}{\lambda}$.
Therefore, by $\s$-admissibility and reproducing kernel property
\begin{align*}
|L(h', z) - L(h, z)| \leq \s |\Delta h(x)| \leq r \s \|\Delta\|_\cH
\leq \frac{\s^2 r^2 \| \bq \|_\infty}{\lambda}
\end{align*}
for all $z = (x,y)$ and this concludes the proof.
\end{proof}

\begin{theorem}
\label{th:uniform}
Let $p \geq 1$ and
$\cF = \set{(\bx, y) \to (\bw \cdot \Psi(\bx) - y)^p \colon
  \|\bw\|_\cH \leq \Lambda}$
where $\cH$ is a Hilbert space and $\Psi\colon \cX \to \cH$ a feature
map.  Assume that the condition $|\bw \cdot \bx - y| \leq M$ holds for
all $(\bx, y)\in \cZ$ and all $\bw$ such that
$\|\bw\|_\cH \leq \Lambda$. Fix $\bq^*$.
Then, if $\bZ_1^T=(\bX_1^T,\bY_1^T)$ is a
sequence of random variables, for any $\d > 0$, with probability at
least $1 - \d$, the following bound holds for all
$h \in H = \set{\bx \to \bw \cdot \Psi(\bx) \colon \|\bw\|_{\cH} \leq
  \Lambda}$ and all $\bq$ such that $0<\|\bq-\bq^*\|_1 \leq 1$:
\begin{align*}
& \E[(h(X_{T+1}) - Y_{T+1})^p|\bZ_1^{T}]
\leq \sum_{t = 1}^T q_t (h(X_{t}) - Y_{t})^p + \disc(\bq) + G(\bq) + 4M \|\bq - \bq^*\|_1
\end{align*}
where $G(\bq) = 4M \Big(\sqrt{2 \log\frac{2}{\delta}} + \sqrt{2 \log\log_2 2 (1 - \|\bq - \bq^*\|_1)^{-1}} + \widetilde{C}_T \Lambda r \Big) \Big(\|\bq^*\|_2 + 2 \|\bq - \bq^*\|_1\Big)$
and $\widetilde{C}_T = 48 pM^p \sqrt{\pi \log T} (1 + 4\sqrt{2} \log^{3/2} (eT^2))$.
Thus,
for $p = 2$,
\begin{align*}
\E[(h(X_{T+1}) &- Y_{T+1})^2|\bZ_1^{T}] 
\leq \sum_{t = 1}^T q_t (h(X_{t}) - Y_{t})^2 + \disc(\bq) \\ & +
O\Bigg( \Lambda r (\log^2 T) \sqrt{\log \log_2 2(1- \|\bq -\bq^*\|_1)^{-1}} \Big(\|\bq^*\|_2 +  \|\bq - \bq^*\|_1 \Big)\Bigg).
\end{align*}
\end{theorem}

This result extends Theorem~\ref{th:lin-hypothesis-bound} 
to hold uniformly over $\bq$. Similarly,
one can prove an analogous extension for Theorem~\ref{th:bound}.
This result suggests that we should try to minimize
$\sum_{t = 1}^T q_t f(Z_t) + \disc(\bq)$ over $\bq$
and $\bw$. This bound is in certain sense analogous to margin bounds:
it is the most favorable when there exists a good choice for $\bq^*$
and we hope to find $\bq$ that is going to be close to this
weight vector. These insights are used to develop our algorithmic solutions
for forecasting non-stationary time series in Section~\ref{sec:algo}. 

\begin{proof}
Let $(\e_k)_{k=0}^\infty$ and
$(\bq(k))_{k=0}^\infty$ be infinite sequences specified below.
By Theorem~\ref{th:lin-hypothesis-bound},
the following holds for each $k$
\begin{align*}
\P \bigg(\E[f(Z_{T + 1})|\bZ_1^T] > 
 \sum_{t = 1}^T q_t(k) f(Z_t)  + \Delta(\bq(k)) + C(\bq(k)) + 4 M
\|\bq\|_2 \e_k    \bigg) \leq  \exp( -\e_k^2 ), 
\end{align*}
where $\Delta(\bq(k))$ denotes the discrepancy computed with respect to
the weights $\bq(k)$ and $C(\bq(k)) = \widetilde{C}_T \|\bq(k)\|_2$.
Let $\e_k =  \e + \sqrt{2 \log k}$. Then, by the
union bound we can write
\begin{align*}
\P\Bigg( \exists k \colon \! \E[f(Z_{T + 1}) |\bZ_1^T]\! >\! 
 \sum_{t = 1}^T q_t(k) f(Z_t)\! +\! \Delta(\bq(k))\! +\! C(\bq(k))\! + \!
4M \|\bq(k)\|_2 \e_k
  \Bigg)  & \!\leq\! \sum_{k=1}^\infty e^{-\e_k^2} \\
  & \!\leq\! \sum_{k=1}^\infty e^{-\e^2 - \log k^2} \\
  & \!\leq\! 2 e^{-\e^2}.
\end{align*}
We choose the sequence $\bq(k)$ to satisfy $\|\bq(k) - \bq^*\|_1 = 1 - 2^{-k}$.
Then, for any $\bq$ such that $0 < \|\bq - \bu\|_1 \leq 1$, there
exists $k \geq 1$ such that
\begin{align*}
1 - \|\bq(k) - \bq^*\|_1 < 1 - \|\bq - \bq^*\|_1 \leq 1-  \|\bq(k-1) - \bq^*\|_1
 \leq  2(1 - \| \bq(k) - \bq^*\|_1).
\end{align*}
Thus, the following inequality holds:
\begin{align*}
\sqrt{2\log k} & \leq \sqrt{2\log \log_2 2(1- \|\bq - \bq^*\|_1)^{-1}}\ignore{ \\
\Leftrightarrow \qquad k^2 & \leq ( \log_2 2 (1 - \|\bq - \bq^*\|_1)^{-1} )^2}.
\end{align*}
Combining this with the observation that the following two
inequalities hold: \ignore{
\begin{align*}
\sR_T(\cF, \bq(k)) \leq M \|\bq(k) - \bq\|_1 + \sR_T(\cF, \bq)
&\leq M\|\bq(k) -\bu\|_1 + M\|\bq - \bu\|_1 +  \sR_T(\cF, \bq) \\
&\leq 2 M\|\bq - \bu\|_1 +  \sR_T(\cF, \bq),
\end{align*}
leads to the following inequality:
\begin{align*}
\e_k \leq & \alpha \Big( \log_2 2\|\bq-\bu\|_1^{-1}  \Big)^2
  \Big(2 M \|\bq-\bu\|_1 + \sR_T(\cF)\Big) 
\\ &+ M (\|\bu\|_2 + \|\bq - \bu\|_1)
  \Big( \e +
       \sqrt{\log\log_2 2\|\bq-\bu\|_1^{-1}}\Big).
\end{align*}
Since we also have}
\begin{align*}
\sum_{t = 1}^T q_t(k-1) f(Z_t) & \leq \sum_{t = 1}^T q_t f(Z_t) + 2 M \|\bq - \bq^*\|_1\\
\Delta(\bq(k-1)) & \leq \Delta(\bq) +  2 M \|\bq - \bq^*\|_1, \\
\|\bq(k-1)\|_2 & \leq 2 \| \bq - \bq^*\|_1 + \|\bq^*\|_2 
\end{align*}
shows that the event
\begin{align*}
\Bigg\{
\E[f(Z_{T + 1})|\bZ_1^T] >
\sum_{t = 1}^T q_t f(Z_t)  + \disc(\bq) 
 + G(\bq) 
+ 4M \|\bq - \bq^*\|_1
\Bigg\}
\end{align*}
where $G(\bq) = 4M \Big(\e + \sqrt{2 \log\log_2 2 (1 - \|\bq - \bq^*\|_1)^{-1}} + \widetilde{C}_T \Lambda r \Big) \Big(\|\bq^*\|_2 + 2 \|\bq - \bq^*\|_1\Big)$
implies the following one
\begin{align*}
\Bigg\{
\E[f(Z_{T + 1}) |\bZ_1^T]\! >\! 
 \sum_{t = 1}^T q_t(k-1) f(Z_t)\! +\! \Delta(\bq(k-1))\! +\! C(\bq(k-1))\! + \!
4M \|\bq(k-1)\|_2 \e_{k-1}
\Bigg\},
\end{align*}
which completes the proof.
\end{proof}

\section{Dual Optimization Problem}
\label{sec:opt}
In this section, we provide a detailed derivation of the optimization
problem in \eqref{eq:convex-dual} starting with optimization
problem in \eqref{eq:joint-opt}.
The first step is to appeal to the following chain of equalities:
\begin{align}
\label{eq:weighted-rreg-duality}
\min_{\bw} &\Big\{ \sum_{t=1}^T q_t (\bw \cdot \Psi(x_t) - y_t)^2 + \lambda_2
\|\bw\|_\cH^2 \Big\} \nonumber\\ &=
\min_{\bw} \Big\{ \sum_{t=1}^T (\bw \cdot x'_t - y'_t)^2 + \lambda_2
\|\bw\|_\cH^2 \Big\} \nonumber\\ &=
\max_{\boldsymbol{\beta}} \Big\{ -\lambda_2 \sum_{t=1}^T \beta_t^2
 - \sum_{s,t=1}^T  \beta_s \beta_t x'_s x'_t +
2 \lambda_2 \sum_{t=1} \beta_t y'_t  \Big\} \nonumber\\ &= 
\max_{\boldsymbol{\beta}} \Big\{ -\lambda_2 \sum_{t=1}^T \beta_t^2
 - \sum_{s,t=1}^T  \beta_s \beta_t \sqrt{q_s}\sqrt{q_t} K_{s,t} +
2 \lambda_2 \sum_{t=1} \beta_t \sqrt{q_t} y_t  \Big\} \nonumber\\ &= 
\max_{\ba} \Big\{ -\lambda_2 \sum_{t=1}^T \frac{\alpha_t^2}{q_t}
 - \ba^T  \bK \ba + 2 \lambda_2 \ba^T \bY  \Big\},
\end{align}
where the first equality follows by substituting $x'_t = \sqrt{q_t} \Psi(x_t)$
and $y'_t = \sqrt{q_t} y_t$ the second equality uses the dual formulation of
the kernel ridge regression problem and the last equality follows from
the following change of variables: $\alpha_t = \sqrt{q_t} \beta_t$.

By~\eqref{eq:weighted-rreg-duality}, optimization problem in
\eqref{eq:joint-opt} is equivalent to the following optimization
problem
\begin{align*}
\min_{0 \leq \bq \leq 1}\bigg\{ \max_{\ba} \Big\{ -\lambda_1 \sum_{t=1}^T \frac{\alpha_t^2}{q_t}
 - \ba^T  \bK \ba + 2 \lambda_1 \ba^T \bY  \Big\} 
  +  (\bd \!\cdot\! \bq) 
+ \lambda_2 \|\bq - \bu\|_p \mspace{-5mu} \bigg\}.
\end{align*}
Next, we apply the change of variables $r_t = 1/q_t$ and appeal to the same
arguments as were given for the primal problem in Section~\ref{sec:algo}
to arrive at \eqref{eq:convex-dual}.

\newpage

\vskip 0.2in
\bibliography{tsj}

\end{document}